%% file: neurips_2026.tex
\documentclass{article}

\PassOptionsToPackage{numbers,compress}{natbib}

\usepackage[main, final]{neurips_2026}

\usepackage[utf8]{inputenc}
\usepackage[T1]{fontenc}
\usepackage{url}
\usepackage{microtype}
\usepackage[dvipsnames,table,xcdraw]{xcolor}
\usepackage{wrapfig}
\usepackage{pifont}

\usepackage{graphicx}
\usepackage{subcaption}
\usepackage{booktabs}
\usepackage{multirow}
\usepackage{dsfont}

\usepackage{algorithm}
\usepackage{algpseudocode}
\usepackage{etoc}
\usepackage{hyperref}
\usepackage{fontawesome5}

\newcommand{\codeurl}{https://github.com/shkabrii/fairsplit}
\newcommand{\codelink}{\expandafter\url\expandafter{\codeurl}}

\usepackage[outline]{contour}
\contourlength{0.6pt}
\usepackage{amsmath}
\usepackage{amssymb}
\usepackage{amsfonts}
\usepackage{mathtools}
\usepackage{amsthm}

\usepackage[capitalize,noabbrev]{cleveref}
\usepackage{enumitem}

\input{abbreviations}
\definecolor{deemph}{gray}{0.6}

\theoremstyle{plain}
\newtheorem{theorem}{Theorem}[section]
\newtheorem{proposition}[theorem]{Proposition}
\newtheorem{lemma}[theorem]{Lemma}

\theoremstyle{definition}
\newtheorem{definition}[theorem]{Definition}

\theoremstyle{remark}

\newcommand{\yes}{\textcolor{ForestGreen}{\ding{51}}}
\newcommand{\no}{\textcolor{BrickRed}{\ding{55}}}
\newcommand{\somewhat}{\textcolor{blue}{\ding{169}}}
\DeclareRobustCommand{\ermstar}{\contour{black}{\textcolor[HTML]{FFA500}{\ding{72}}}}
\DeclareRobustCommand{\idealplus}{\ding{58}}

\title{\algname{}: Decomposing the Embedding Space for Fair Classification}

\author{
Andrii Shkabrii\textsuperscript{1,2}\thanks{Corresponding author: \texttt{andrii.shkabrii@univie.ac.at}}\quad
Anna Beer\textsuperscript{3,1}\quad
Ylli Sadikaj\textsuperscript{1}\quad
Claudia Plant\textsuperscript{1,4}\\
\textsuperscript{1}Faculty of Computer Science, University of Vienna, Vienna, Austria \\
\textsuperscript{2}Doctoral School Computer Science, University of Vienna, Vienna, Austria \\
\textsuperscript{3}Webster Vienna Private University, Vienna, Austria \\
\textsuperscript{4}Research Network Data Science, University of Vienna, Vienna, Austria
}

\begin{document}
\etocdepthtag.toc{mtmain}

\maketitle

\begin{abstract}
Fair classification faces a twofold challenge: maintaining high classification accuracy while limiting reliance on sensitive information. Our method \algname{} achieves this by learning a decomposition of the embedding space into a Fair Space used for prediction and a Sensitive Space that isolates sensitive information. During training, the model must simultaneously learn the classification task and separate fair from sensitive information. The Sensitive Space gives sensitive attributes a destination, so the Fair Space can be used for prediction without relying on them. This decomposition is guided via a learnable mask and the Hilbert-Schmidt Independence Criterion (HSIC) as a measure of statistical dependence. We provide theoretical guarantees showing that the demographic parity gap is upper-bounded by both (i) the learned dimensionality of the Fair Space and (ii) the HSIC between the Fair Space and sensitive attributes. \algname{} effectively handles multiclass tasks and non-binary sensitive attributes, overcoming key limitations of existing fair classification methods. Experiments on established fairness benchmarks show that \algname{} achieves competitive accuracy-fairness trade-offs. \href{\codeurl}{\textcolor{black}{\faGithub}\ \texttt{Codebase}}
\end{abstract}

\section{Introduction}
\label{sec:intro}

With the expanding use of machine learning (ML) models in high-stakes domains, including hiring~\citep{raghavan2020mitigating}, medicine~\citep{feng2024fair}, criminal justice~\citep{dataset_compas}, credit scoring~\citep{EOppEOdd}, and face recognition~\citep{wang2019racial,gong2020jointly,caldeira2024mst,neto2025balancing}, numerous studies have shown that these systems can perpetuate historical biases affecting underrepresented demographic groups. Consequently, group fairness has gained attention in ML tasks like classification~\citep{feng2024fair,fairnessOverfitting}.
However, state-of-the-art methods for fair classification remain limited: many methods are designed for binary classification or assume binary sensitive attributes, restricting their applicability to realistic multiclass settings~\citep{DiffMCDP_MCDP_ICML, DiffABCC_ABCC_TMLR, DRAlign_FAIRER, fairMixup_iclr}.
While empirical improvements in fairness metrics are frequently reported, theoretical guarantees are scarce.
Most in-processing methods for fair classification regularize dependence between predictions and sensitive attributes during training. However, acting at the prediction level provides limited guarantees for the resulting classifier and does not directly shape the learned representation it operates on.

\begin{figure}[t]
    \centering
    \includegraphics[width=\linewidth]{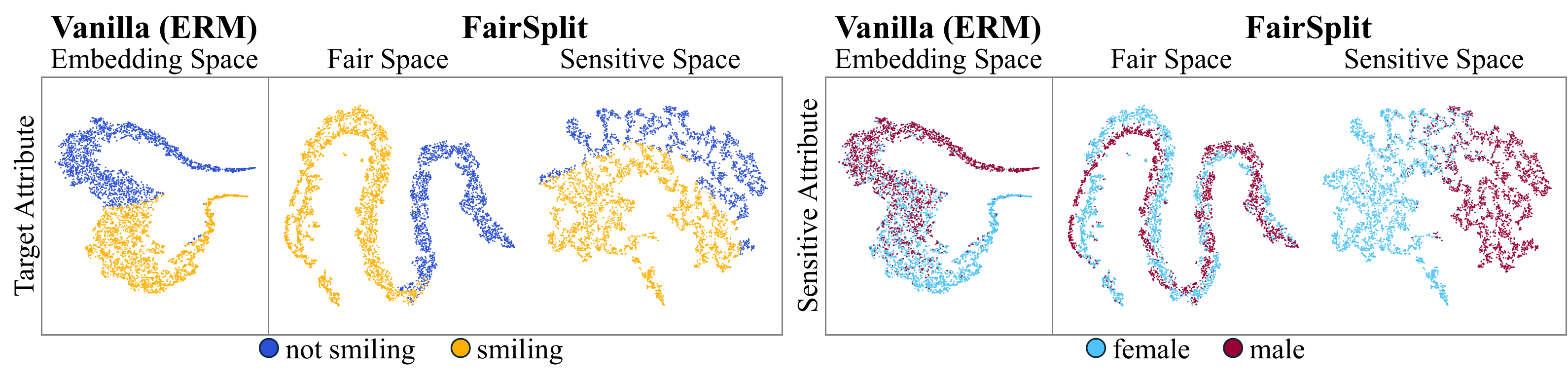}
    \caption{\textbf{Embedding spaces on CelebA-S}, where the task is to predict whether a person is smiling and gender is the sensitive attribute, i.e., the protected characteristic on which predictions should not depend, visualized with t-SNE and colored by the target (left) and the sensitive attribute (right). Each side shows three panels: the Vanilla embedding, trained by empirical risk minimization (ERM) without fairness constraints, and the Fair Space and Sensitive Space of \algname{}. In the Vanilla (ERM) embedding, \textit{smiling}/\textit{not smiling} correlates with \textit{female}/\textit{male}. \algname{} instead classifies from the Fair Space, which is guided towards independence from the sensitive attribute.}
    \label{fig:embedding-space}
\end{figure}

We introduce \algname{}, an in-processing method for fair classification that acts on the representation instead: it learns which embedding dimensions to use for prediction and which to exclude because they carry sensitive information (\Cref{fig:embedding-space}). Specifically, our method splits the embedding space into (i) a Fair Space with minimized dependence on the sensitive attribute using the Hilbert-Schmidt Independence Criterion (HSIC)~\cite{gretton2005measuring}, and (ii) a complementary Sensitive Space that captures residual sensitive information and is excluded from the classifier. The dimensionality of the Fair Space is itself learned during training rather than fixed in advance. We prove that a continuous relaxation of the demographic parity gap is upper-bounded by both this dimensionality and the HSIC between the Fair Space and the sensitive attribute, both of which our method directly controls.
The bound holds for multiclass classification and non-binary sensitive attributes, and applies to the downstream classifier operating on the learned representation. We validate \algname{} on standard benchmarks across multiple modalities, demonstrating competitive accuracy and strong fairness performance. \algname{} succeeds in multiclass settings with non-binary sensitive attributes, whereas many state-of-the-art baselines are limited to binary settings by design, a capability gap we highlight in \Cref{tab:method_capability_summary}.

Our main contributions are as follows:
\begin{enumerate}[label=\(\bullet\), leftmargin=5mm,nosep] 
\itemsep0.15cm
\item We propose \algname{}, an in-processing method for fair classification that decomposes the embedding into a Fair Space used for prediction and a Sensitive Space, learning the dimensionality of the Fair Space jointly with the classifier rather than fixing it in advance.
\item We prove an upper bound on a continuous relaxation of the demographic parity gap that supports multiclass and non-binary sensitive attribute scenarios, and show that \algname{} controls this bound by reducing both the dimensionality of the Fair Space and the HSIC between the Fair Space and sensitive attributes.
\item We empirically validate our theoretical results on established fair classification benchmarks, demonstrating strong performance in both binary settings and multiclass scenarios where most existing in-processing methods cannot be applied.
\end{enumerate}

\section{Related Work}\label{sec:related_work}

\textbf{Fairness Notions and Evaluation:} Fairness is inherently multifaceted and depends on the perspective and application context. In this work, we focus on \emph{group-level fairness}, with an emphasis on \emph{demographic parity}, which requires the outcome of an ML model to be independent of a given sensitive attribute, so that group membership does not influence the prediction.
A central and widely adopted measure for this is the demographic parity gap, $\Delta$DP~\cite{feldman2015certifying}, which can be computed on model logits, prediction probabilities, or the final class assignments. 
Several recent works have proposed refinements and complementary measures, such as MCDP~\cite{DiffMCDP_MCDP_ICML} and ABCC~\cite{DiffABCC_ABCC_TMLR}, which capture fairness violations continuously without reliance on an explicit decision threshold. Other prominent fairness notions include Equal Opportunity (EOpp) and Equalized Odds (EOdd)~\cite{EOppEOdd}, which compare error rates across sensitive groups rather than prediction rates. We report all five metrics in our experiments, with demographic parity and its refinements (ABCC, MCDP) as our primary objectives.

\textbf{Fair Classification and Representation Learning:} Fairness-aware learning methods aim to satisfy one or more fairness notions while maintaining predictive performance. They are commonly divided into pre-processing, in-processing, and post-processing approaches. Pre-processing methods modify the training data, for example by reweighing samples or repairing features, so that a classifier trained on it is less biased~\cite{kamiran2012data,feldman2015certifying}. Post-processing methods instead adjust the predictions of a trained model~\cite{postprocessing_dp_icml,multiclass_postprocessing,multiclass_DP}, which often requires the sensitive attribute at test time~\cite{EOppEOdd}.

\begin{wraptable}{r}{0.46\textwidth}
\centering
\caption{\textbf{Method capabilities.} Support for multiclass targets and non-binary sensitive attributes ($\nngroups > 2$). \yes~supported, \no~binary only, \somewhat~indicates that it should be possible, but has not been tested in the respective papers.}
\label{tab:method_capability_summary}
\small
\setlength{\tabcolsep}{9pt}
\begin{tabular}{lcc}
\toprule
Method & Multiclass & $\nngroups > 2$ \\
\midrule
DRAlign~\cite{DRAlign_FAIRER}      & \no       & \no       \\
DiffMCDP~\cite{DiffMCDP_MCDP_ICML} & \no       & \no       \\
PRemover~\cite{premover}           & \no       & \no       \\
MMD-B-Fair~\cite{mmd-b-fair}       & \no       & \no       \\
FairMixup~\cite{fairMixup_iclr}    & \somewhat & \no       \\
AdvDebias~\cite{AdvDebias}         & \somewhat & \somewhat \\
HSIC$(Y,S)$~\cite{HSIC_ypred_s_journal} & \somewhat & \somewhat \\
\midrule
\rowcolor{gray!15}\algname{} (ours) & \yes & \yes \\
\bottomrule
\end{tabular}
\vspace{-8pt}
\end{wraptable}
In this work, we focus on \emph{in-processing} methods, which add fairness objectives directly to training. There is a broad spectrum of in-processing approaches. In our experiments (\Cref{sec:experiments}), we compare our method to representative techniques from several major categories: 
adversarial approaches against a sensitive-attribute predictor (e.g., AdvDebias~\cite{AdvDebias}),
augmentation-based regularizers (e.g., FairMixup~\cite{fairMixup_iclr}), and 
gap-regularization methods (e.g., DiffMCDP~\cite{DiffMCDP_MCDP_ICML}). 
Furthermore, we compare to DRAlign~\cite{DRAlign_FAIRER}, which constrains intermediate layers by aligning decision rationales across groups, and several methods that, similar to \algname{}, enforce statistical independence between representations or predictions and sensitive attributes (e.g., MMD-B-Fair~\cite{mmd-b-fair}, HSIC$(Y,S)$~\cite{HSIC_ypred_s_journal}, and PRemover~\cite{premover}).

\textbf{Independence and Space Splitting for Fairness:} Our work relates to methods that learn representations explicitly designed to be independent of sensitive attributes. Early approaches such as LFR~\cite{learningFairRepresentations_early} and PRemover~\cite{premover} discourage dependence via parity constraints and mutual information. Other methods enforce independence using statistical tests (e.g., MMD-B-Fair~\cite{mmd-b-fair}), or the Rényi correlation~\cite{renyi}.
The Hilbert-Schmidt Independence Criterion (HSIC) can detect non-linear, statistical dependencies~\cite{gretton2005measuring}. Despite these properties, HSIC has only rarely been used for fair classification~\citep{ffb}. Existing HSIC-based approaches primarily enforce independence at the prediction level: HSIC$(Y,S)$~\cite{HSIC_ypred_s_journal} penalizes dependence between predicted labels and sensitive attributes, and FKL~\cite{HSIC_ypred_s_fair_regression_DR} applies HSIC in a regression setting with a focus on dimensionality reduction. Beyond classification, FairIB~\cite{FairIB_HSIC_fair_recommender} uses HSIC in fair recommender systems. HSIC has also been used to decompose the embedding: building on the mask-based decomposition of~\citet{acedec}, H-SPLID~\cite{H-SPLID_NeurIPS} uses HSIC to split the embedding into task-relevant and task-irrelevant features, and bounds the change in predictions under input perturbations by the dimensionality of the task-relevant space and its HSIC with the input. \algname{} brings this idea to fair classification, with a Fair Space used for prediction and a Sensitive Space that isolates sensitive information. A separate line of variational disentanglement methods~\cite{fair_disentanglement_VAE, sarhan2020fairness, FADES} also decomposes the embedding space, but within a variational autoencoder aimed at fair generation or counterfactual synthesis rather than classification.

\textbf{Open Challenges:} Despite significant progress, several challenges in in-processing fair classification remain. As \Cref{tab:method_capability_summary} shows, most existing methods are limited to binary sensitive attributes and binary classification, and do not naturally extend to multiclass or non-binary sensitive groups. Recent benchmarks~\cite{ffb} reflect this limitation. Yet multiclass classification and multi-group sensitive attributes are increasingly relevant in practice~\cite{multiclass_postprocessing,multiclass_scoring,dcp_fairness_evaluation_eodd_based,fairnessOverfitting,multiclass_DP}. Addressing these settings in a principled manner remains an open problem, which we take up in this work.

\section{Methodology}
\label{sec:methodology}
\algname{} trains a classifier together with a decomposition of its embedding into a Fair Space, which the classifier uses for prediction, and a Sensitive Space, which absorbs sensitive information. We first formalize the setting (\Cref{sec:problem_setup}) and then introduce the training objective that shapes the two spaces (\Cref{sec:decomposition}). We next describe how the mask separating them is learned jointly with the network (\Cref{sec:algorithm}), and finally bound the demographic parity gap of the resulting classifier (\Cref{sec:theoretical_guarantees}).

\subsection{Problem Setup}
\label{sec:problem_setup}
We consider a classification task over a dataset $\mathcal{D} = \{(\vx_i, \y_i, \sens_i)\}_{i=1}^N$, where $\vx_i \in \xspace \subseteq \real^d$ denotes the input, $\y_i \in \yspace = \{1, \dots, K\}$ the class label, and $\sens_i \in \sspace = \{1, \dots, \nngroups\}$ the sensitive attribute with $\nngroups$ groups. We denote their one-hot encodings by $\vy_i \in \{0,1\}^K$ and $\vs_i \in \{0,1\}^{\nngroups}$. The model consists of an encoder $\encoder: \xspace \to \zspace$ that maps each input to an embedding $\vz_i = \encoder(\vx_i)$ in the embedding space $\zspace \subseteq \real^m$, and a classifier $\classifier: \zspace \to \real^K$ that outputs logits, from which class probabilities are obtained with a softmax. We denote all model parameters by $\allparams = \{\encparams, \clfparams\}$. For reference, \Cref{tab:notation} in Appendix~\ref{app:notation} summarizes the notation used throughout the paper.

\subsection{Learning a Fair Decomposition}
\label{sec:decomposition}
To achieve a decomposition of the embedding space into a Fair and a Sensitive Space while preserving high utility, we leverage three mechanisms: structural compression (Equation~\ref{eq:comp_loss}), dependence regularization via HSIC (Equation~\ref{eq:hsic_loss}), and a utility objective for predictive performance (Equation~\ref{eq:task_loss}). \Cref{fig:fairsplit-overview} provides an overview of our approach.
\begin{figure}[t]
    \centering
    \includegraphics[width=\linewidth]{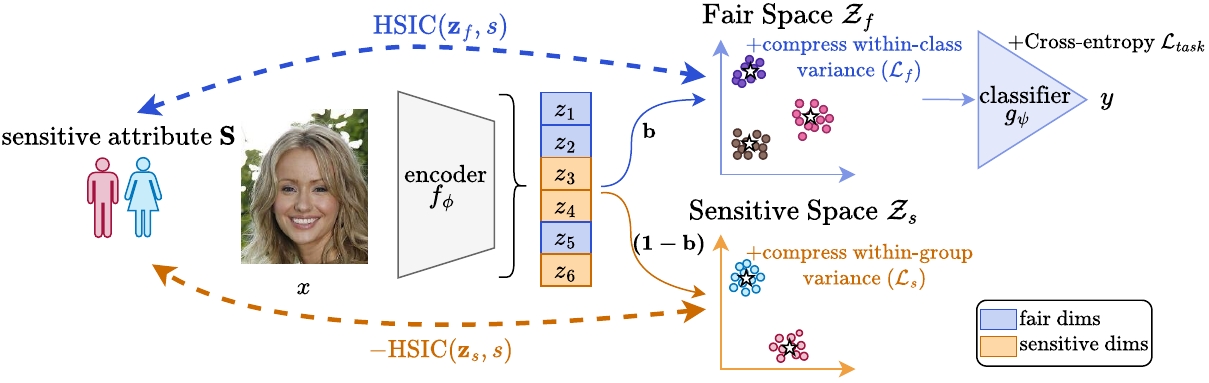}
    \caption{\textbf{Overview of \algname{}.} An encoder maps each input to an embedding $\vz$, and a learned binary mask $\vb$ splits it into a Fair Space $\zfair = \vb \odot \vz$ and a Sensitive Space $\zsens = (\mathbf{1} - \vb) \odot \vz$. Only the Fair Space is passed to the classifier. During training, HSIC regularization reduces the dependence between the Fair Space and the sensitive attribute and increases it in the Sensitive Space, which gives sensitive information a destination outside the classifier's input. The mask is updated in closed form by comparing, per dimension, how strongly the embedding varies within classes and within sensitive groups, so that the size of the Fair Space is learned end-to-end rather than fixed in advance.}
    \label{fig:fairsplit-overview}
\end{figure}
To decompose the embedding $\vz = \encoder(\vx)$, we introduce a binary mask $\vb \in \{0,1\}^m$. A binary mask provides a clear separation: each dimension belongs either to the Fair Space or the Sensitive Space. This enables the classifier to operate on a subset of features that is regularized toward independence from the sensitive attribute. The Fair Space is defined by $\zfair=\vb \odot \vz$, and the complementary Sensitive Space by $\zsens = (\mathbf{1} - \vb) \odot \vz$. The mask, and with it the dimensionality of the Fair Space $\dfair = \|\vb\|_0$, is learned during training rather than fixed in advance.~\looseness=-1

\paragraph{Structural Compression.}
Learning a binary mask requires a signal that indicates which dimensions should be assigned to which space. The structural compression loss (Equation~\ref{eq:comp_loss}) provides this signal by measuring how well each dimension aligns with class labels versus sensitive groups. Dimensions where within-class variance is low relative to within-group variance are assigned to the Fair Space. Conversely, dimensions with low within-group variance are assigned to the Sensitive Space.

We align the Fair Space with the class labels and the Sensitive Space with the sensitive attribute. This is achieved by regularizing the embedding space to compress class-relevant information into the Fair Space and sensitive information into the Sensitive Space (see \Cref{fig:fairsplit-overview}).

Let $\mathcal{I}_k$ and $\mathcal{I}_g$ denote the index sets for class $k$ and sensitive group $g$, with centroids:
\begin{equation}
    \cclass_k = \frac{1}{|\mathcal{I}_k|} \sum_{i \in \mathcal{I}_k} \vz_i, \qquad \cgroup_g = \frac{1}{|\mathcal{I}_g|} \sum_{i \in \mathcal{I}_g} \vz_i.
\end{equation}
The structural compression loss sums a Fair Space term $\mathcal{L}_f$ and a Sensitive Space term $\mathcal{L}_s$:
\begin{equation}
\label{eq:comp_loss}
    \losscomp(\mathcal{D};\encparams,\vb) = \mathcal{L}_f + \mathcal{L}_s,
\end{equation}
where $\mathcal{L}_f = \lamf \sum_{k} \sum_{i \in \mathcal{I}_k} \norm{\vb \odot (\vz_i - \cclass_k)}_2^2$ and $\mathcal{L}_s = \lams \sum_{g} \sum_{i \in \mathcal{I}_g} \norm{(\mathbf{1}-\vb) \odot (\vz_i - \cgroup_g)}_2^2$. The term $\mathcal{L}_f$ compresses within-class variance in the Fair Space, so the classifier learns to rely on class-specific information. The term $\mathcal{L}_s$ compresses within-group variance in the Sensitive Space to identify during training which features carry sensitive information and should be excluded from prediction. The hyperparameters $\lamf$ and $\lams$ control the relative importance of each objective and influence the mask allocation.

\paragraph{HSIC Regularization.}
While the structural compression objective $\losscomp$ aligns the mask $\vb$, it does not guarantee statistical independence. We address this by regularizing the learned subspaces via the Hilbert-Schmidt Independence Criterion (HSIC)~\citep{gretton2005measuring}, a kernel-based measure of statistical dependence. For characteristic kernels (e.g., Gaussian RBF), $\HSIC(U, V) = 0$ iff $U \perp V$, making it well-suited for enforcing independence. Given $n$ samples, HSIC can be estimated empirically as:
\begin{equation}
    \widehat{\HSIC}(U, V) = (n-1)^{-2} \operatorname{tr}(\mat{K}_u \mat{H} \mat{K}_v \mat{H}),
\end{equation}
where $\mat{K}_u, \mat{K}_v$ are kernel matrices and $\mat{H} = \mat{I} - \frac{1}{n}\mathbf{1}\mathbf{1}^\top$ is the centering matrix (see Appendix~\ref{app:preliminaries} for details). For a batch of $n$ samples, let $\mat{Z} = [\vz_1, \ldots, \vz_n]^\top$ and $\mat{S} = [\vs_1, \ldots, \vs_n]^\top$, and let $\mat{Z}_{\tagFair}$ and $\mat{Z}_{\tagSens}$ denote the masked versions of $\mat{Z}$, with rows $\vb \odot \vz_i$ and $(\mathbf{1} - \vb) \odot \vz_i$. We then apply HSIC with opposing objectives to the two spaces:

\begin{align}
\label{eq:hsic_loss}
    \lossreg(\mathcal{D};\encparams,\vb) = \rhof \widehat{\HSIC}(\mat{Z}_{\tagFair},\mat{S})
    - \rhos \widehat{\HSIC}(\mat{Z}_{\tagSens},\mat{S})
\end{align}
The first term reduces statistical dependence between the Fair Space and the sensitive attribute $\sens$. The second term keeps sensitive information in $\zsens$, preventing it from leaking into the Fair Space.

\paragraph{Utility Objective.}
The decomposition serves a dual purpose during training: the model must simultaneously learn the classification task and separate fair from sensitive information. The Sensitive Space $\zsens$ is essential for this process. By learning an explicit space to capture sensitive information, it gives sensitive information a destination outside the Fair Space, reducing the dependence of $\zfair$ on the sensitive attribute. The classifier $\classifier$ operates solely on $\zfair$, so the full model is $\model(\vx) = \classifier(\vb \odot \encoder(\vx))$, and the utility objective is:
\begin{equation}
\label{eq:task_loss}
\losstask(\mathcal{D}; \allparams, \vb) = \frac{1}{N} \sum_{i=1}^N \ell(\vy_i, \classifier(\vb \odot \vz_i)),
\end{equation}
where $\ell$ denotes the standard cross-entropy loss~\cite{goodfellow2016deep}.

\paragraph{Total Objective.}
Combining the three components, we obtain the full training objective:
\begin{equation}\label{eq:loss_full}
    \losstotal(\mathcal{D}; \allparams, \vb) = \losstask + \losscomp + \lossreg
\end{equation}
Each component serves a distinct purpose: $\losstask$ maintains predictive performance, $\losscomp$ provides structural compression for the mask, $\lossreg$ guides the statistical dependence.

\subsection{The \algname{} Algorithm}
\label{sec:algorithm}
Since $\vb$ is binary, the mask cannot be learned by gradient descent. We therefore optimize the model parameters and the mask alternately, as described in Algorithm~\ref{alg:fairsplit} (Appendix~\ref{app:algorithm}). First, we fix the mask and update the network parameters $\allparams$ with Adam, while updating the centroids $\cclass$ and $\cgroup$ with a moving average, so that the targets of the compression loss change smoothly across minibatches. Then, we fix the representations and update the mask by minimizing the two terms of $\losscomp$ with equal weights, so that the assignment of a dimension does not depend on how strongly $\lamf$ and $\lams$ regularize the encoder toward the centroids. By relaxing the binary constraint to the interval $[0,1]$, this sub-problem admits the following closed-form solution:

\begin{proposition}[Closed-Form Mask Solution]
\label{prop:mask_solution}
\begin{equation}\label{eq:mask_solution}
    b_j^* = \frac{V_j^{(\vartagGroup)}}{V_j^{(\vartagClass)} + V_j^{(\vartagGroup)}}, \quad \forall j \in \{1, \dots, m\},
\end{equation}
where $V_j^{(\vartagClass)} = \sum_{k=1}^K \sum_{i \in \mathcal{I}_k} (\symbZ_{ij} - \cclass_{kj})^2$ and $V_j^{(\vartagGroup)} = \sum_{g=1}^{\nngroups} \sum_{i \in \mathcal{I}_g} (\symbZ_{ij} - \cgroup_{gj})^2$ are the within-class and within-group variances, and $\symbZ_{ij}$, $\cclass_{kj}$, $\cgroup_{gj}$ are the $j$-th entries of $\vz_i$, $\cclass_k$, $\cgroup_g$.
\end{proposition}

\noindent The derivation is in Appendix~\ref{app:proof_mask}. Since Equation~\ref{eq:mask_solution} is evaluated on each minibatch and yields values in $[0,1]$, we also smooth the mask with an exponential moving average and binarize it at $0.5$.

\subsection{Theoretical Analysis}
\label{sec:theoretical_guarantees}

We prove that a continuous relaxation of the demographic parity gap is bounded by the dimensionality of the Fair Space and the $\HSIC$ between the Fair Space and the sensitive attribute, for any number $\nngroups \geq 2$ of sensitive groups. For a model $\model$, we generalize the continuous relaxation of the demographic parity gap~\citep{learningFairRepresentations_early,fairMixup_iclr,DiffABCC_ABCC_TMLR} to $K$ classes and $\nngroups$ groups:
\begin{equation}
    \ddp(\model) := \max_{g, g'} \big\| \expect[\model(\vx) \mid \sens=g] - \expect[\model(\vx) \mid \sens=g'] \big\|_2 .
\end{equation}
It measures the largest Euclidean distance between the mean logits of any two sensitive groups. By operating on the continuous outputs, it captures bias in the model's behavior before a final class is assigned. In our experiments, $\Delta$DP denotes the gap in predicted classes (Appendix~\ref{app:metrics}).

\begin{theorem}[Demographic Parity Bound]
\label{thm:dp-hsic}
Let $p_{\min} = \min_g \mathbb{P}(\sens=g) > 0$ denote the minimum group prevalence.
Let $\model = \classifier \circ (\vb \odot \encoder)$ be a masked classifier where
$\encoder: \xspace \to \real^m$ is an encoder,
$\vb \in \{0,1\}^m$ is a binary mask with $\dfair = \|\vb\|_0$ active dimensions,
and $\classifier(\vz) = \mat{W}\vz$ is a linear classification head with $K$ outputs and
largest absolute entry $\|\mat{W}\|_{\max} \leq B$.
Assume the Fair Space $\zfair = \vb \odot \encoder(\vx)$ satisfies $\expect[\norm{\zfair}_2] < \infty$.
Then the demographic parity gap $\ddp(\model)$ is bounded as:
\begin{equation}
\label{eq:dp_bound}
    \ddp(\model) \leq c_\nngroups B\sqrt{K\dfair} \cdot
    \sqrt{\HSIC(\zfair, \vs)},
\end{equation}
where $\HSIC$ uses linear kernels on both arguments (Definition~\ref{def:linear_hsic}) and $c_\nngroups \leq \sqrt{2}/p_{\min}$ is a constant that depends on the number of groups and their prevalences (Equation~\ref{eq:cG}).
\end{theorem}

Intuitively, Theorem~\ref{thm:dp-hsic} implies that demographic parity is jointly constrained by the statistical dependence ($\sqrt{\HSIC}$) and the dimensionality of the Fair Space ($\sqrt{\dfair}$). This motivates \algname{}'s dual strategy: minimizing $\lossreg$ reduces the dependence between the Fair Space and the sensitive attribute, while the learnable mask $\vb$ moves dimensions that carry sensitive information to the Sensitive Space, reducing $\dfair$. In practice, $\lossreg$ estimates HSIC with a Gaussian kernel: since the kernel is characteristic, driving the first term of $\lossreg$ to zero enforces $\zfair \perp \sens$, under which the dependence term in Equation~\ref{eq:dp_bound} vanishes. The fairness-utility trade-off is empirically stable across smooth kernel families and across a range of bandwidths (see Appendix~\ref{app:kernel_sensitivity}). We observe that a smaller $\dfair$ goes along with better fairness across multiple metrics in Figure~\ref{fig:dim-vs-fairness}.

\paragraph{Proof Sketch.}
The proof (Appendix~\ref{app:proof_main_theorem}) proceeds by first linking the demographic parity gap to the HSIC of the model outputs (Lemma~\ref{lem:dp-cov-ident}). We then exploit the masked linear classification head structure $\model(\vx) = \mat{W}(\vb \odot \vz)$ to decouple this dependence into a weight norm term and a representation covariance term. The factor $\sqrt{\dfair}$ arises from bounding the Frobenius norm of $\mat{W}$ restricted to active dimensions. Since every quantity in the bound can be computed from the trained model and the data, it can be evaluated after training (Appendix~\ref{app:linear_bound_validation}).

\subsection{Limitations and Discussion}
\label{sec:limitations}
Our theoretical bound applies to classifiers with linear heads, covering standard architectures like ResNets~\citep{he2016deep}, ViTs~\citep{dosovitskiy2020image}, and MLPs, while extending guarantees to non-linear heads remains open. In line with this linear scope, we observe that a non-linear probe can still recover the sensitive attribute from the Fair Space (Appendix~\ref{app:probes}). Empirically, however, \algname{} retains substantial fairness gains when the linear head is replaced by a multi-layer non-linear classifier on CelebA-A and UTKFace-AR (Appendix~\ref{app:nonlinear_head}). The linear-head assumption applies only to the classifier: the encoder $\encoder$ can be arbitrarily non-linear. \algname{} applies to fine-tuned ViTs and frozen foundation models (\Cref{sec:pretrained}) and transfers across five pretrained encoders spanning dimensionalities from $512$ to $2048$ without per-architecture retuning, with the mask adapting $\dfair$ to each embedding size (Appendix~\ref{app:encoder_robustness}). Our bound scales with the inverse of the minimum group prevalence $p_{\min}$, reflecting the known challenge of fairness under imbalance~\citep{brzezinski2024properties}, yet \algname{} retains its advantage on intersectional attributes with $p_{\min}\!=\!0.05$ (\Cref{tab:utkface_compas}) and under imposed imbalance (Appendix~\ref{app:imbalance}). More generally, when the target and the sensitive attribute are dependent, any classifier that reduces the demographic parity gap below the difference between the groups' base rates must incur a minimum error~\citep{zhao2022inherent}, a trade-off we quantify for each benchmark in \Cref{sec:binary_results}. A practical limitation is the requirement for sensitive attributes during training, though \algname{} is partly robust to noisy or missing sensitive labels (Appendix~\ref{app:noisy_missing}), and future work could explore semi-supervised approaches when annotations are unavailable. Finally, the Sensitive Space retains information about the sensitive attribute, so applications that must keep it private should protect the model parameters.

\section{Experiments}\label{sec:experiments}
We follow the evaluation protocol of the in-processing fair classification benchmark by~\citet{ffb}, averaging experiments over five seeds each, and extend it to accommodate scenarios with multiclass classification and non-binary sensitive attributes. Code is available at \codelink{}.

\paragraph{Datasets.}
For binary classification experiments, we use the Adult~\citep{dataset_adult} tabular dataset and the CelebA-A~\citep{dataset_celeba} image dataset. In Adult, the task is to predict whether an individual's income exceeds \$50K, with gender (male/female) as the binary sensitive attribute. In CelebA-A, we predict the attractive attribute from face images, using gender as the binary sensitive attribute. To evaluate performance in multiclass settings with non-binary sensitive attributes, we use UTKFace-AR~\citep{dataset_utkface}, where the task is to classify age into three groups (Young: $\leq$25, Middle-aged: 26-45, and Old: $>$45) from face images, with race as the sensitive attribute grouped into three categories (White, Black, and Other). Finally, to evaluate intersectional fairness under severe imbalance, we use the COMPAS recidivism dataset~\citep{dataset_compas} with an intersectional sensitive attribute (gender $\times$ age) yielding four groups with $p_{\min}\!=\!0.05$.  For the experiments with large pretrained models (\Cref{sec:pretrained}), we additionally use FairFace~\citep{karkkainen2021fairface} and CelebA~\citep{dataset_celeba} at $224\times224$ resolution. Further details are provided in Appendix~\ref{app:datasets}. 
\paragraph{Comparison Methods.}
We compare \algname{} against several representative in-processing fair classification baselines: empirical risk minimization (ERM), the adversarial approach AdvDebias~\cite{AdvDebias}, the augmentation-based method FairMixup~\cite{fairMixup_iclr}, the alignment method DRAlign~\cite{DRAlign_FAIRER}, the independence-based methods HSIC$(Y,S)$~\cite{HSIC_ypred_s_journal}, PRemover~\cite{premover}, and MMD-B-Fair~\cite{mmd-b-fair}, as well as the gap-regularization method DiffMCDP~\cite{DiffMCDP_MCDP_ICML}. We note that many of these methods are restricted to binary classification with binary sensitive attributes by design (see Table~\ref{tab:method_capability_summary}).

\paragraph{Evaluation Metrics.}
We report Average Precision (AP$\uparrow$) for utility in binary settings and F1 score (F1$\uparrow$) for multiclass settings, along with several fairness metrics, including demographic parity gap ($\Delta$DP$\downarrow$), equal opportunity ($\Delta$EOpp$\downarrow$), equalized odds ($\Delta$EOdd$\downarrow$), ABCC$\downarrow$~\cite{DiffABCC_ABCC_TMLR}, and MCDP$\downarrow$~\cite{DiffMCDP_MCDP_ICML}. Our theoretical guarantee (Theorem~\ref{thm:dp-hsic}) concerns demographic parity, while $\Delta$EOpp and $\Delta$EOdd compare error rates rather than prediction rates across groups and are not covered by it, so we report them as complementary empirical metrics. For multiclass targets and non-binary sensitive attributes, we extend these metrics by taking the maximum disparity across all classes and all pairs of sensitive groups, following~\citet{multiclass_DP} and~\citet{multiclass_postprocessing}. Note that in fair classification, predictive accuracy is not the sole criterion. When training labels reflect historical biases, a model that perfectly reproduces them may be accurate yet unfair~\cite{HSIC_ypred_s_journal}. Hence, we aim to balance accuracy and fairness by considering only models that achieve at least $95\%$ of ERM's validation utility (AP/F1) and comparing their fairness, following prior work~\citep{jung2023re, DiffMCDP_MCDP_ICML}. See Appendix~\ref{app:metrics} for more details. Following~\citet{ffb}, to characterize the utility-fairness trade-off rather than a single operating point, we additionally visualize trade-off curves by varying each method's fairness-controlling hyperparameter (see Figure~\ref{fig:fairsplit-tradeoffs}). Each curve connects a method's points in order of this hyperparameter, so a curve changes direction wherever a fairness metric does not vary monotonically with it.

\subsection{Comparison and Trade-Offs in Binary Settings}
\label{sec:binary_results}

\begin{table}[t]
\caption{\textbf{Performance on binary settings.} CelebA-A and Adult are binary classification tasks with gender as the sensitive attribute, and we measure utility by average precision (AP). For each method, we consider all configurations that reach at least $95\%$ of ERM's validation AP and report the one with the lowest $\Delta$DP, without further optimizing AP. \textbf{Bold} and \underline{underline} mark the best and second-best results. Results are averaged over 5 seeds. $\Delta\pi$ is the base-rate gap between groups (\Cref{sec:binary_results}).}
\label{tab:combined_performance}
\centering
\setlength{\tabcolsep}{3pt}
\resizebox{\textwidth}{!}{
\begin{tabular}{l rrrrr|r!{\vrule width 0.8pt} rrrrr|r}
\toprule
  & \multicolumn{6}{c!{\vrule width 0.8pt}}{CelebA-A ($\Delta\pi = 40.0$)} & \multicolumn{6}{c}{Adult ($\Delta\pi = 19.8$)} \\
\cmidrule(lr){2-7} \cmidrule(lr){8-13}
 & $\Delta$DP $\downarrow$ & ABCC $\downarrow$ & MCDP $\downarrow$ & $\Delta$EOpp $\downarrow$ & $\Delta$EOdd $\downarrow$ & AP $\uparrow$ & $\Delta$DP $\downarrow$ & ABCC $\downarrow$ & MCDP $\downarrow$ & $\Delta$EOpp $\downarrow$ & $\Delta$EOdd $\downarrow$ & AP $\uparrow$ \\
\midrule
ERM & 52.8$\pm$0.8 & 35.9$\pm$0.8 & 53.4$\pm$0.8 & 36.8$\pm$1.5 & 75.0$\pm$2.1 & \underline{84.5$\pm$0.4} & 19.4$\pm$0.6 & 18.7$\pm$0.2 & 36.6$\pm$0.6 & \underline{9.3$\pm$2.0} & \underline{18.2$\pm$2.3} & \textbf{76.7$\pm$0.5} \\
\midrule
AdvDebias & 45.9$\pm$3.4 & 31.2$\pm$1.6 & 46.8$\pm$3.8 & 26.0$\pm$4.4 & 57.1$\pm$7.6 & \textbf{84.9$\pm$0.7} & 6.1$\pm$5.7 & 7.7$\pm$5.0 & 16.7$\pm$11.1 & 21.8$\pm$9.5 & 25.3$\pm$8.9 & 72.8$\pm$2.6 \\
FairMixup & 37.2$\pm$1.1 & 19.5$\pm$0.9 & 39.5$\pm$1.2 & 19.5$\pm$2.4 & 39.8$\pm$2.5 & 83.1$\pm$0.1 & 1.1$\pm$0.6 & 4.6$\pm$0.5 & 17.2$\pm$2.4 & 32.6$\pm$2.8 & 36.0$\pm$3.3 & 72.9$\pm$0.4 \\
DRAlign & 31.7$\pm$4.9 & 15.4$\pm$2.9 & 33.6$\pm$4.3 & 12.4$\pm$4.8 & 29.9$\pm$9.1 & 80.3$\pm$1.2 & \textbf{0.7$\pm$0.5} & 6.1$\pm$0.3 & 19.8$\pm$0.9 & 34.2$\pm$1.0 & 37.8$\pm$1.0 & 73.6$\pm$0.4 \\
HSIC$(Y,S)$ & 23.9$\pm$1.4 & 13.6$\pm$0.5 & 27.0$\pm$1.7 & 3.3$\pm$1.5 & 11.6$\pm$3.0 & 81.3$\pm$0.5 & 1.9$\pm$0.5 & \underline{4.1$\pm$0.3} & \underline{11.2$\pm$1.1} & 27.8$\pm$1.7 & 31.1$\pm$2.0 & 73.0$\pm$0.4 \\
DiffMCDP & \underline{20.4$\pm$2.4} & 14.4$\pm$1.6 & \underline{22.3$\pm$2.2} & 3.2$\pm$1.0 & \underline{9.8$\pm$2.5} & 80.6$\pm$1.0 & 3.0$\pm$0.4 & 4.3$\pm$0.4 & \textbf{11.0$\pm$1.1} & 25.8$\pm$1.6 & 28.4$\pm$1.7 & 73.6$\pm$0.2 \\
PRemover & 25.1$\pm$1.2 & \underline{13.4$\pm$0.5} & 26.7$\pm$1.2 & \underline{3.1$\pm$1.7} & 13.5$\pm$3.0 & 80.6$\pm$0.2 & 3.6$\pm$0.5 & 7.1$\pm$0.4 & 19.1$\pm$1.1 & 24.7$\pm$1.3 & 26.9$\pm$1.6 & \underline{74.5$\pm$0.4} \\
MMD-B-F. & 37.1$\pm$13.3 & 24.2$\pm$10.5 & 37.8$\pm$12.9 & 17.9$\pm$13.6 & 42.7$\pm$25.0 & 80.4$\pm$4.6 & 3.7$\pm$0.8 & 5.2$\pm$1.7 & 15.0$\pm$5.5 & 24.7$\pm$2.8 & 26.8$\pm$3.4 & 73.1$\pm$1.0 \\
\midrule
\rowcolor{gray!15}\textbf{\algname{}} & \textbf{18.3$\pm$1.1} & \textbf{9.4$\pm$1.3} & \textbf{19.6$\pm$1.8} & \textbf{2.6$\pm$0.3} & \textbf{5.2$\pm$1.7} & 80.4$\pm$1.1 & \underline{1.0$\pm$0.9} & \textbf{2.4$\pm$0.7} & 19.8$\pm$4.5 & \textbf{8.2$\pm$4.1} & \textbf{8.3$\pm$4.2} & 72.9$\pm$0.9 \\
\bottomrule
\end{tabular}
}
\end{table}

\begin{table}[t]
\caption{\textbf{Performance on multiclass and intersectional settings.} UTKFace-AR has a 3-class age target and 3 racial groups, with F1 as utility. COMPAS has a binary target and an intersectional gender~$\times$~age attribute, with AP as utility. Methods restricted to binary targets or sensitive attributes are omitted (\Cref{tab:method_capability_summary}). As in \Cref{tab:combined_performance}, all configurations achieve at least $95\%$ of ERM's validation utility, and \textbf{bold} and \underline{underline} mark the best and second-best results. Results are averaged over 5 seeds. With more than two groups, $\Delta\pi$ is the largest base-rate gap over pairs of groups and classes.}
\label{tab:utkface_compas}
\centering
\setlength{\tabcolsep}{3pt}
\resizebox{\textwidth}{!}{
\begin{tabular}{l rrrrr|r!{\vrule width 0.8pt} rrrrr|r}
\toprule
  & \multicolumn{6}{c!{\vrule width 0.8pt}}{UTKFace-AR ($\Delta\pi = 25.9$)} & \multicolumn{6}{c}{COMPAS ($\Delta\pi = 26.6$)} \\
\cmidrule(lr){2-7} \cmidrule(lr){8-13}
 & $\Delta$DP $\downarrow$ & ABCC $\downarrow$ & MCDP $\downarrow$ & $\Delta$EOpp $\downarrow$ & $\Delta$EOdd $\downarrow$ & F1 $\uparrow$ & $\Delta$DP $\downarrow$ & ABCC $\downarrow$ & MCDP $\downarrow$ & $\Delta$EOpp $\downarrow$ & $\Delta$EOdd $\downarrow$ & AP $\uparrow$ \\
\midrule
ERM & 27.9$\pm$1.4 & 27.1$\pm$2.6 & 36.2$\pm$2.4 & 29.1$\pm$5.2 & 35.9$\pm$3.0 & \underline{73.8$\pm$0.7} & 35.2$\pm$3.3 & 17.9$\pm$1.8 & 58.8$\pm$4.5 & \underline{24.9$\pm$4.7} & 55.4$\pm$8.7 & \textbf{69.1$\pm$2.4} \\
\midrule
AdvDebias & 28.7$\pm$3.0 & 21.1$\pm$1.9 & 32.9$\pm$2.5 & 32.7$\pm$7.8 & 38.5$\pm$7.9 & \textbf{74.7$\pm$0.8} & 26.4$\pm$7.7 & 13.9$\pm$3.1 & 47.8$\pm$10.9 & 29.5$\pm$5.7 & 44.7$\pm$12.4 & \underline{68.4$\pm$2.1} \\
HSIC$(Y,S)$ & \textbf{11.9$\pm$1.3} & \underline{12.0$\pm$0.8} & \underline{20.4$\pm$2.4} & \underline{18.1$\pm$2.7} & \underline{24.1$\pm$3.5} & 70.3$\pm$1.1 & \underline{20.5$\pm$3.1} & \underline{6.7$\pm$0.7} & \underline{35.8$\pm$4.6} & 28.3$\pm$9.2 & \underline{36.1$\pm$9.8} & 67.9$\pm$2.3 \\
\midrule
\rowcolor{gray!15}\textbf{\algname{}} & \underline{12.6$\pm$1.1} & \textbf{9.1$\pm$1.0} & \textbf{18.7$\pm$1.2} & \textbf{14.9$\pm$4.6} & \textbf{22.1$\pm$4.9} & 70.4$\pm$1.5 & \textbf{16.8$\pm$3.6} & \textbf{1.2$\pm$0.4} & \textbf{32.9$\pm$6.2} & \textbf{23.8$\pm$6.7} & \textbf{31.5$\pm$7.6} & 66.8$\pm$2.4 \\
\bottomrule
\end{tabular}
}
\end{table}

\begin{table}[t]
\centering
\caption{\textbf{\algname{} on large pretrained models} (5 seeds, gender as the sensitive attribute). ERM uses the same architecture as \algname{}. ViT-B/16 is fine-tuned end-to-end on FairFace to predict whether a person is at least 30 years old. For the frozen backbones, the Fair Space is learned for predicting high cheekbones on CelebA. A new linear head is then trained on this fixed Fair Space to predict a different attribute, age. $\dfair / m$ is the size of the Fair Space relative to the full embedding.}
\label{tab:pretrained}
\small
\setlength{\tabcolsep}{2.5pt}
\begin{tabular}{l l cccccc}
\toprule
Backbone & Method & AP $\uparrow$ & $\Delta$DP $\downarrow$ & ABCC $\downarrow$ & $\Delta$EOpp $\downarrow$ & $\Delta$EOdd $\downarrow$ & $\dfair / m$ \\
\midrule
ViT-B/16 & ERM & \textbf{88.9$\pm$0.2} & 23.0$\pm$0.5 & 17.2$\pm$0.1 & 22.3$\pm$1.1 & 33.5$\pm$1.1 & 768 / 768 \\
\textit{(fine-tuned)} & \cellcolor{gray!15}\algname{} & \cellcolor{gray!15}86.4$\pm$0.3 & \cellcolor{gray!15}\textbf{9.5$\pm$1.6} & \cellcolor{gray!15}\textbf{5.8$\pm$1.4} & \cellcolor{gray!15}\textbf{8.1$\pm$2.6} & \cellcolor{gray!15}\textbf{9.5$\pm$1.9} & \cellcolor{gray!15}552 / 768 \\
\midrule
CLIP ViT-L/14 & ERM & \textbf{96.5$\pm$0.1} & 21.4$\pm$0.9 & 20.0$\pm$0.6 & 4.9$\pm$0.5 & 23.6$\pm$2.4 & 256 / 256 \\
\textit{(frozen)} & \cellcolor{gray!15}\algname{} & \cellcolor{gray!15}93.1$\pm$0.4 & \cellcolor{gray!15}\textbf{9.1$\pm$1.3} & \cellcolor{gray!15}\textbf{6.4$\pm$0.7} & \cellcolor{gray!15}\textbf{0.8$\pm$0.3} & \cellcolor{gray!15}\textbf{4.8$\pm$1.6} & \cellcolor{gray!15}174 / 256 \\
\midrule
CLIP ResNet-50 & ERM & \textbf{96.6$\pm$0.1} & 23.4$\pm$0.4 & 21.0$\pm$0.3 & 6.1$\pm$0.5 & 29.4$\pm$1.5 & 256 / 256 \\
\textit{(frozen)} & \cellcolor{gray!15}\algname{} & \cellcolor{gray!15}93.6$\pm$0.4 & \cellcolor{gray!15}\textbf{12.2$\pm$2.4} & \cellcolor{gray!15}\textbf{7.9$\pm$1.5} & \cellcolor{gray!15}\textbf{0.9$\pm$0.5} & \cellcolor{gray!15}\textbf{8.5$\pm$5.0} & \cellcolor{gray!15}168 / 256 \\
\midrule
DINOv2 ViT-L/14 & ERM & \textbf{95.2$\pm$0.1} & 23.5$\pm$0.4 & 22.1$\pm$0.2 & 7.3$\pm$0.3 & 35.3$\pm$1.7 & 256 / 256 \\
\textit{(frozen)} & \cellcolor{gray!15}\algname{} & \cellcolor{gray!15}90.9$\pm$0.5 & \cellcolor{gray!15}\textbf{9.1$\pm$1.8} & \cellcolor{gray!15}\textbf{5.9$\pm$1.2} & \cellcolor{gray!15}\textbf{0.5$\pm$0.4} & \cellcolor{gray!15}\textbf{7.1$\pm$1.7} & \cellcolor{gray!15}166 / 256 \\
\bottomrule
\end{tabular}
\end{table}
\begin{figure*}[t]
    \centering
    \includegraphics[width=\linewidth]{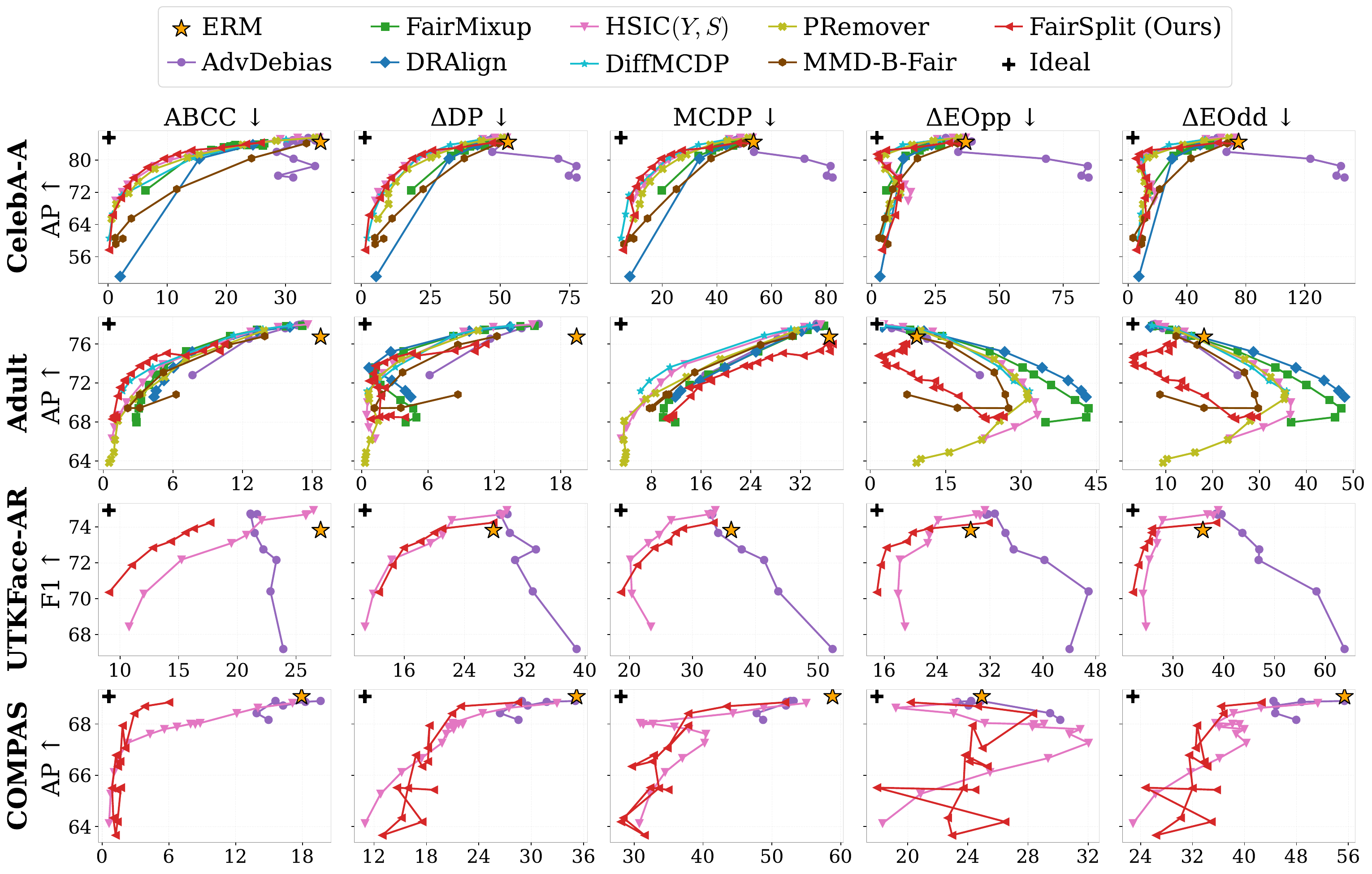}
    \caption{\textbf{Trade-off between performance and fairness} (higher performance is better) across four datasets. Each row is one dataset (CelebA-A, Adult, UTKFace-AR, COMPAS) with its own y-axis (AP for binary classification, F1 for the multiclass UTKFace-AR), and each column is one fairness metric. ERM (\ermstar) is a single point. The ideal model (\idealplus) lies in the top-left corner of every panel. \algname{} matches or outperforms competing methods across most settings and achieves the best ABCC on every dataset in \Cref{tab:combined_performance,tab:utkface_compas}.\label{fig:fairsplit-tradeoffs}}
\end{figure*}
Although most baselines are designed for binary classification with binary sensitive attributes, \algname{} matches or exceeds them on both datasets (\Cref{tab:combined_performance}, \Cref{fig:fairsplit-tradeoffs}). \algname{} achieves the lowest ABCC on Adult and CelebA-A and the lowest MCDP on CelebA-A. A low ABCC means the per-group score distributions are close on average across decision thresholds, while MCDP measures the gap at the single worst threshold. On Adult, predicted scores cluster in a narrow range below $0.5$, so small per-seed differences inflate MCDP while ABCC stays small, which is why methods regularizing predictions directly lead on MCDP there. On $\Delta$EOpp and $\Delta$EOdd, \algname{} achieves the best results on both datasets. While theoretical trade-offs exist between demographic parity and equalized odds~\citep{kleinberg2016inherent}, \algname{} attains a more efficient trade-off in our experiments. More broadly, part of the fairness-utility trade-off is inherent to the data, because it depends on how much the groups' base rates differ. We denote this gap by $\Delta\pi = |\mathbb{P}(\y{=}1 \mid \sens{=}1) - \mathbb{P}(\y{=}1 \mid \sens{=}2)|$ and report it for each benchmark in \Cref{tab:combined_performance,tab:utkface_compas}. \citet{zhao2019inherent,zhao2022inherent} show that a classifier with a $\Delta$DP below $\Delta\pi$ cannot have a low error on both groups, since its two group-wise error rates sum to at least $\Delta\pi - \Delta$DP. Once $\Delta$DP exceeds $\Delta\pi$, this lower bound becomes negative and no longer constrains the classifier. On CelebA-A ($\Delta\pi = 40.0$), ERM and AdvDebias both have a $\Delta$DP above the gap ($52.8$ and $45.9$), so the bound does not prevent AdvDebias from improving fairness and AP at the same time. \algname{} instead reduces $\Delta$DP to $18.3$, well below the gap, which the bound allows only at a minimum error. In a controlled experiment that changes only $\Delta\pi$, the AP cost of \algname{} increases from $0.9$ to $6.7$ as $\Delta\pi$ grows from $5$ to $45$ points (Appendix~\ref{app:basegap}).

\subsection{Multiclass, Non-Binary, and Intersectional Settings}
The embedding-decomposition mechanism behind \algname{} is agnostic to the number of classes $K$ and groups $\nngroups$, so the same algorithm applies to multiclass targets and non-binary sensitive attributes without modification, a property only AdvDebias and HSIC$(Y,S)$ share among our baselines. \Cref{tab:utkface_compas} reports results on UTKFace-AR and COMPAS, the latter being the regime where Theorem~\ref{thm:dp-hsic}'s $1/p_{\min}$ factor makes the bound loosest. On both, \algname{} outperforms on most fairness metrics and is competitive on the rest, with utility within the $95\%$-of-ERM threshold. On COMPAS, \algname{} also attains the lowest ABCC ($1.2$ against $6.7$ for the next-best method). The $1/p_{\min}$ factor stems from the pairwise definition of demographic parity (Lemma~\ref{lem:dp-cov-ident}), yet \algname{}'s empirical fairness does not degrade at the rate the bound predicts. The imposed-imbalance experiment in Appendix~\ref{app:imbalance} confirms this gap between guarantee and behavior persists across imbalance regimes.

\subsection{Vision Transformers and Frozen Foundation Models}
\label{sec:pretrained}
\algname{} applies unchanged to large pretrained models (\Cref{tab:pretrained}). When \algname{} fine-tunes a ViT-B/16~\citep{dosovitskiy2020image} end-to-end on FairFace at $224 \times 224$, it lowers $\Delta$DP from $23.0$ to $9.5$ and more than halves every other fairness metric, while retaining $97\%$ of ERM's AP. We also apply \algname{} to three frozen backbones, CLIP ViT-L/14, CLIP ResNet-50, and DINOv2 ViT-L/14~\citep{radford2021clip,he2016deep,oquab2023dinov2}, training a small projection on top of the features and learning the Fair Space for one CelebA attribute. A new linear head trained on this Fair Space, without sensitive labels, then predicts a related, unseen attribute, and improves every fairness metric on all three backbones while retaining at least $95\%$ of ERM's AP. Appendix~\ref{app:scaling} describes the setup and reports results for both tasks.

\subsection{Controlling Fairness via the Sensitive Space}

\begin{figure}[t]
\centering
\begin{minipage}[t]{0.40\linewidth}
    \vspace{0pt}
    \centering
    \includegraphics[width=0.95\linewidth]{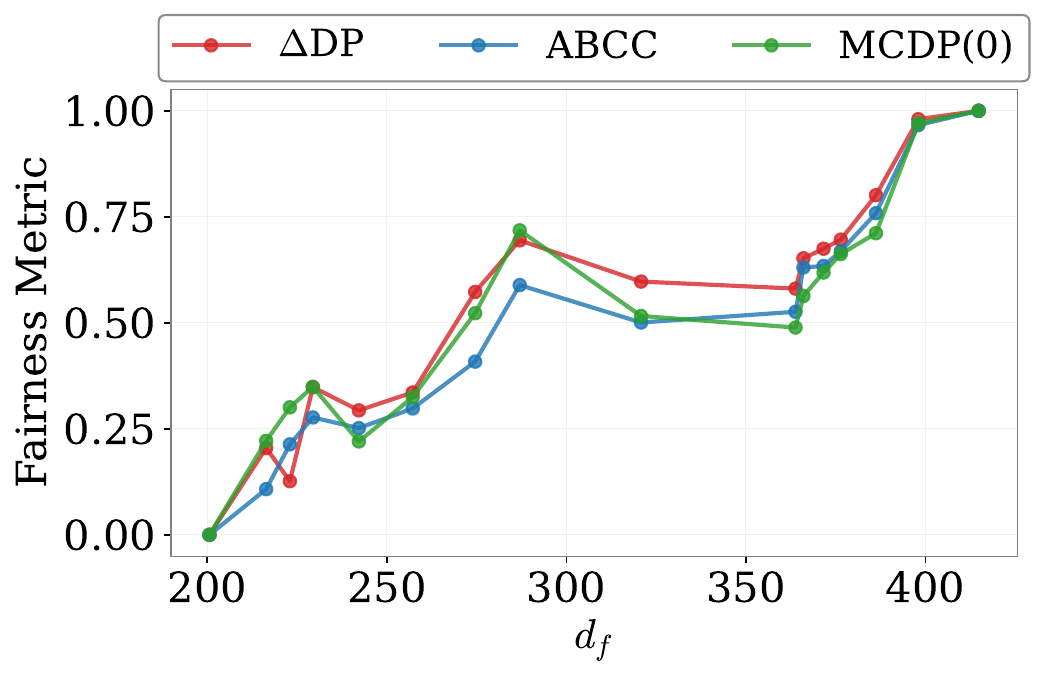}
    \captionof{figure}{\textbf{Fairness vs. Fair Space dim.} Smaller $\dfair$ generally goes along with better fairness across $\Delta$DP, ABCC, and MCDP. Metrics are normalized to $[0,1]$.}
    \label{fig:dim-vs-fairness}
\end{minipage}\hfill
\begin{minipage}[t]{0.58\linewidth}
    \vspace{0pt}
    \centering
    \captionof{table}{\textbf{Ablation of the Sensitive Space} (\mbox{CelebA-A}, 5 seeds). Penalizing HSIC on the full embedding (\emph{No Sensitive Space}) halves ERM's $\Delta$DP, adding the mask without the Sensitive Space losses (\emph{Unregularized}, $\lams\!=\!\rhos\!=\!0$) reduces it further, and full \algname{} yields the lowest $\Delta$DP and ABCC with the smallest Fair Space at comparable AP. More metrics and datasets are in \Cref{tab:component_ablation_extra} in the appendix.}
    \label{tab:component_ablation}
    \small
    \setlength{\tabcolsep}{2.5pt}
    \begin{tabular}{l cccc}
    \toprule
    Variant & AP $\uparrow$ & $\Delta$DP $\downarrow$ & ABCC $\downarrow$ & $\dfair$ \\
    \midrule
    ERM & \textbf{84.5$\pm$0.4} & 52.8$\pm$0.8 & 35.9$\pm$0.8 & 512 \\
    No Sensitive Space & 80.7$\pm$0.7 & 26.6$\pm$1.4 & 16.1$\pm$1.0 & 512 \\
    Unregularized & 80.1$\pm$0.8 & 22.2$\pm$2.4 & 12.4$\pm$1.7 & 391 \\
    \rowcolor{gray!15}\algname{} & 80.4$\pm$1.1 & \textbf{18.3$\pm$1.1} & \textbf{9.4$\pm$1.3} & 334 \\
    \bottomrule
    \end{tabular}
\end{minipage}
\end{figure}

Theorem~\ref{thm:dp-hsic} establishes a connection between the dimensionality of the Fair Space $\dfair$ and fairness: a smaller $\dfair$ tightens the upper bound on the demographic parity gap. To validate this empirically, we train on CelebA-A with varying $\rhos$, which controls how strongly sensitive information is pushed into the Sensitive Space. Figure~\ref{fig:dim-vs-fairness} aligns with our theoretical finding. Higher $\rhos$ yields smaller $\dfair$, and models with lower $\dfair$ generally achieve better fairness, with the same trend visible on the threshold-free metrics ABCC and MCDP. Maximizing the dependence between $\zsens$ and the sensitive attribute via $\rhos$ encourages the model to move sensitive information out of the Fair Space, so the Fair Space becomes smaller without fixing $\dfair$ in advance, which tightens the bound in Theorem~\ref{thm:dp-hsic}.

\subsection{Ablations and Hyperparameter Sensitivity}
\label{sec:ablation}

\textbf{Why an explicit Sensitive Space?} A natural question is whether the Sensitive Space is necessary, or whether suppressing dependence on $\sens$ in a single embedding via HSIC would suffice. \Cref{tab:component_ablation} answers this by introducing the components of \algname{} one at a time on CelebA-A. When HSIC is applied to the full embedding without any decomposition, it acts on the entire input of the classifier, so sensitive information can only be suppressed, not moved to a separate space. This lowers $\Delta$DP from $52.8$ for ERM to $26.6$, which remains well above the $18.3$ of \algname{}. Learning the mask separates the dimensions that mainly encode the target from those that mainly encode the sensitive attribute, so the classifier only sees the Fair Space. This already improves fairness, before any loss term acts on the Sensitive Space ($\lams\!=\!\rhos\!=\!0$). Regularizing the Sensitive Space then actively moves sensitive information into it, which shrinks the Fair Space further and gives the lowest $\Delta$DP at comparable AP. This matches the role of $\dfair$ in Theorem~\ref{thm:dp-hsic} as a structural regularizer of the bound, and probes confirm that the Sensitive Space retains the sensitive attribute (Appendix~\ref{app:probes}). The same pattern holds on UTKFace-AR (Appendix~\ref{app:component_ablation_extra}).\\

\textbf{How sensitive is \algname{} to its hyperparameters?} We vary each of the four loss weights on CelebA-A while keeping the others fixed (Appendix~\ref{app:hp_sensitivity}, \Cref{fig:ablations}). Fairness depends mostly on the independence weight $\rhof$: a larger $\rhof$ clearly improves $\Delta$DP, ABCC, and MCDP, while AP decreases only gradually. The weight $\lamf$ of the Fair Space has little effect on the results. The two weights of the Sensitive Space, $\lams$ and $\rhos$, mainly decide how many dimensions the mask assigns to the Sensitive Space, and larger values leave fewer dimensions in the Fair Space. Finally, the results remain stable for different kernel types and for different kernel bandwidths (Appendix~\ref{app:kernel_sensitivity}). \algname{} also tolerates partly missing sensitive labels: on CelebA-A, $\Delta$DP rises to $28.0$ with $30\%$ of labels missing, versus $52.8$ for ERM, while incorrect labels are harder, and $30\%$ of them remove the gain (Appendix~\ref{app:noisy_missing}).\looseness=-1

\section{Conclusion and Future Work}\label{sec:conclusion}
We propose \algname{}, a fair classification method that decomposes the embedding space into two subspaces: the Sensitive Space provides an explicit destination for sensitive information during training, while the Fair Space is isolated for prediction. Unlike methods that regularize predictions directly, \algname{} regularizes dependence at the representation level via HSIC, enabling a theoretical guarantee on demographic parity. Our experiments demonstrate fairness improvements across multiple metrics, including in multiclass settings with non-binary sensitive attributes, and these improvements extend to a fine-tuned ViT-B/16 and to frozen foundation models, where a Fair Space learned for one task keeps its fairness gains on a related, unseen task (\Cref{sec:pretrained}). By maintaining an isolated Sensitive Space rather than discarding sensitive information, \algname{} provides additional flexibility in choosing whether sensitive information should be included or discarded depending on the downstream task. A natural next step is to extend the theoretical guarantee beyond linear heads, using the characteristic kernels that \algname{} already applies in its HSIC regularization.

\begin{ack}
We want to thank the open source community, particularly the developers of PyTorch~\citep{Pytorch}, NumPy~\citep{numpy}, scikit-learn~\citep{sklearn_api}, and Matplotlib~\citep{matplotlib}, and the authors of FFB~\citep{ffb}, MCDP~\citep{DiffMCDP_MCDP_ICML}, and MMD-B-Fair~\citep{mmd-b-fair} for releasing their code.
\end{ack}

\bibliographystyle{abbrvnat}
\bibliography{library}

\clearpage
\appendix
\etocdepthtag.toc{mtappendix}
\etocsettagdepth{mtmain}{none}
\etocsettagdepth{mtappendix}{subsection}
\etocsettocstyle{\section*{Appendix}\vspace{-0.5\baselineskip}}{}
\tableofcontents
\newpage

\section{Notation Summary}
\label{app:notation}

\begin{table}[h!]
\centering
\caption{Summary of notation used in this paper.}
\label{tab:notation}
\renewcommand{\arraystretch}{0.95}
\begin{tabular}{ll}
\toprule
\textbf{Symbol} & \textbf{Description} \\
\midrule
\multicolumn{2}{l}{\textbf{Data and Labels}} \\
$\vx, \vx_i$ & Input features \\
$\y, \y_i$ & Class label $\in \{1, \ldots, K\}$ \\
$\vy, \vy_i$ & One-hot encoding of class label \\
$\sens, \sens_i$ & Sensitive attribute $\in \{1, \ldots, \nngroups\}$ \\
$\vs, \vs_i$ & One-hot encoding of sensitive attribute \\
$s_g$ & Indicator $\mathds{1}[\sens = g]$ for group $g$ \\
\midrule
\multicolumn{2}{l}{\textbf{Spaces, Sets, and Dimensions}} \\
$\xspace$ & Input space $\subseteq \real^d$ \\
$\zspace$ & Embedding space $\subseteq \real^m$ \\
$\mathcal{I}_k$ & Index set for samples with class label $k$ \\
$\mathcal{I}_g$ & Index set for samples with sensitive attribute $g$ \\
$N, n$ & Number of training samples, batch size \\
$\nngroups$ & Number of sensitive groups \\
$K$ & Number of classes \\
$m$ & Embedding dimensionality \\
$\dfair$ & Dimensionality of Fair Space ($= \|\vb\|_0$) \\
\midrule
\multicolumn{2}{l}{\textbf{Model Components}} \\
$\encoder$ & Encoder $f_\phi: \xspace \to \zspace$ \\
$\classifier$ & Classifier $g_\psi: \zspace \to \real^K$ \\
$\model$ & Full model $h_\theta = \classifier \circ (\vb \odot \encoder)$ \\
$\vb$ & Binary mask $\in \{0,1\}^m$ \\
$\vz$ & Embedding $= \encoder(\vx)$ \\
$\zfair$ & Fair Space $= \vb \odot \vz$ \\
$\zsens$ & Sensitive Space $= (\mathbf{1} - \vb) \odot \vz$ \\
\midrule
\multicolumn{2}{l}{\textbf{Batch Notation}} \\
$\mat{Z}$ & Batch of embeddings $\in \real^{n \times m}$ \\
$\Zfair, \Zsens$ & Batch in the Fair Space and in the Sensitive Space \\
$\mat{S}$ & Batch of one-hot sensitive attributes $\in \{0,1\}^{n \times \nngroups}$ \\
\midrule
\multicolumn{2}{l}{\textbf{Centroids and Means}} \\
$\cclass_k$ & Class centroid in full $\zspace$ (before masking) \\
$\cgroup_g$ & Group centroid in full $\zspace$ (before masking) \\
$\meanemb_g, \bar{\meanemb}$ & Group and overall mean output (proofs) \\
\midrule
\multicolumn{2}{l}{\textbf{Loss and Hyperparameters}} \\
$\losstask, \losscomp, \lossreg, \losstotal$ & Task, compression, regularization, and total losses \\
$\lamf, \lams$ & Compression loss weights \\
$\rhof, \rhos$ & HSIC regularization weights \\
$\sigma$ & Bandwidth of the Gaussian kernel \\
\midrule
\multicolumn{2}{l}{\textbf{Mask}} \\
$V_j^{(\vartagClass)}$ & Within-class variance along dim $j$ \\
$V_j^{(\vartagGroup)}$ & Within-group variance along dim $j$ \\
$b_j^*$ & Closed-form mask value for dim $j$ \\
$\alpha$ & Mask moving-average rate \\
\midrule
\multicolumn{2}{l}{\textbf{Fairness and Theory}} \\
$\ddp$ & Demographic parity gap of the logits (theory) \\
$p_g$ & Prevalence of group $g$: $\mathbb{P}(\sens = g)$ \\
$p_{\min} = p_{(1)} \leq p_{(2)}$ & Smallest and second-smallest group prevalences \\
$\Delta\pi$ & Gap between the groups' positive base rates \\
$c_\nngroups$ & Group constant of the DP bound (Equation~\ref{eq:cG}) \\
$\mat{W}, B$ & Weight matrix of the linear head and bound on its entries \\
\bottomrule
\end{tabular}
\end{table}

\section{Technical Preliminaries}
\label{app:preliminaries}

Here, we provide the definitions regarding the Hilbert-Schmidt Independence Criterion (HSIC) that are used in our method and in the subsequent proofs.

\subsection{Hilbert-Schmidt Independence Criterion (HSIC)}

The Hilbert-Schmidt Independence Criterion~\citep{gretton2005measuring} is a kernel-based measure of statistical dependence.

\begin{definition}[HSIC]
\label{def:hsic}
Let $\mathcal{H}_u$ and $\mathcal{H}_v$ be reproducing kernel Hilbert spaces (RKHS) of functions on domains $\mathcal{U}$ and $\mathcal{V}$, with kernels $k_u$ and $k_v$. For random variables $U \in \mathcal{U}$ and $V \in \mathcal{V}$, the HSIC is defined as
\begin{equation}
    \HSIC(U, V) := \|C_{UV}\|_{\mathrm{HS}}^2,
\end{equation}
where $C_{UV} = \expect\big[(k_u(U,\cdot) - \expect k_u(U,\cdot)) \otimes (k_v(V,\cdot) - \expect k_v(V,\cdot))\big]$ is the cross-covariance operator and $\|\cdot\|_{\mathrm{HS}}$ denotes the Hilbert-Schmidt norm.
\end{definition}

If both kernels are characteristic (e.g., Gaussian RBF), $\HSIC(U, V) = 0$ if and only if $U \perp V$.

\paragraph{Empirical Estimation.}
Given $n$ i.i.d.\ samples $\{(u_i, v_i)\}_{i=1}^n$, HSIC can be estimated empirically~\citep{gretton2005measuring} as:
\begin{equation}
    \widehat{\HSIC}(U, V) = (n-1)^{-2} \operatorname{tr}(\mat{K}_u \mat{H} \mat{K}_v \mat{H}),
\end{equation}
where $\mat{K}_u, \mat{K}_v \in \real^{n \times n}$ are kernel matrices with elements $(\mat{K}_u)_{ij} = k_u(u_i, u_j)$ and $(\mat{K}_v)_{ij} = k_v(v_i, v_j)$, and $\mat{H} = \mat{I} - \frac{1}{n}\mathbf{1}\mathbf{1}^\top$ is the centering matrix. On the embeddings we use a Gaussian kernel $k_u(\vct{u}, \vct{u}') = \exp\big(-\|\vct{u} - \vct{u}'\|_2^2 / (2\sigma^2 q)\big)$ for $\vct{u} \in \real^q$, with bandwidth $\sigma$, and on the one-hot sensitive attribute $\vs$ a linear kernel. In practice, following prior work~\citep{H-SPLID_NeurIPS, wang2021revisiting}, we use the normalized cross-covariance operator (NOCCO)~\citep{fukumizu_cond} to obtain a normalized HSIC penalty.

\begin{definition}[HSIC with Linear Kernels]
\label{def:linear_hsic}
Let $\vct{u} \in \real^q$ be a random vector and let $\vs$ be the one-hot encoding of $\sens$. With linear kernels on both spaces, HSIC reduces to:
\begin{equation}
    \HSIC(\vct{u}, \vs) := \|\mat{C}\|_F^2 = \sum_{g=1}^\nngroups \|\mathrm{Cov}(\vct{u}, s_g)\|_2^2,
\end{equation}
where $\mat{C} = \mathrm{Cov}(\vct{u}, \vs) \in \real^{q \times \nngroups}$ is the cross-covariance matrix, and $s_g = \mathds{1}[\sens = g]$ is the $g$-th component of $\vs$. Throughout our theoretical analysis, $\HSIC(\cdot, \cdot)$ denotes this instantiation.
\end{definition}

\section{Proof of Proposition~\ref{prop:mask_solution} (Closed-Form Mask)}
\label{app:proof_mask}

\begin{proof}
Since the mask applies element-wise, the two terms of $\losscomp$ with equal weights decompose across dimensions. For each $j \in \{1, \dots, m\}$, we solve:
\begin{equation}
    \min_{b_j \in [0,1]} \; b_j^2 V_j^{(\vartagClass)} + (1 - b_j)^2 V_j^{(\vartagGroup)},
\end{equation}
where $V_j^{(\vartagClass)} = \sum_{k=1}^K \sum_{i \in \mathcal{I}_k} (\symbZ_{ij} - \cclass_{kj})^2$ and $V_j^{(\vartagGroup)} = \sum_{g=1}^\nngroups \sum_{i \in \mathcal{I}_g} (\symbZ_{ij} - \cgroup_{gj})^2$ are the within-class and within-group variances of dimension $j$. The first measures how much dimension $j$ varies among samples of the same class, the second how much it varies among samples of the same sensitive group. Let $J(b_j)$ denote this objective. Taking the derivative:
\begin{equation}
    \frac{\partial J}{\partial b_j} = 2 b_j V_j^{(\vartagClass)} - 2 (1 - b_j) V_j^{(\vartagGroup)}.
\end{equation}
Setting to zero and solving for $b_j$:
\begin{equation}
    b_j V_j^{(\vartagClass)} = (1 - b_j) V_j^{(\vartagGroup)} \quad \Rightarrow \quad b_j^* = \frac{V_j^{(\vartagGroup)}}{V_j^{(\vartagClass)} + V_j^{(\vartagGroup)}}.
\end{equation}
Since $V_j^{(\vartagClass)}, V_j^{(\vartagGroup)} \geq 0$ are not both zero, the objective is strictly convex, so $b_j^*$ is its unique minimizer, and since $b_j^* \in [0,1]$, it is also the minimizer over $[0,1]$. Thus $b_j^*$ is close to $1$, assigning dimension $j$ to the Fair Space, when $V_j^{(\vartagClass)}$ is small relative to $V_j^{(\vartagGroup)}$.
\end{proof}

\section{Proof of Theorem~\ref{thm:dp-hsic} (Main Result)}
\label{app:proof_main_theorem}

In this section, we prove the demographic parity bound (Theorem~\ref{thm:dp-hsic}). The proof follows two steps:
\begin{enumerate}
    \item We first connect the demographic parity gap $\ddp$ of an arbitrary predictor to the HSIC between its outputs and the sensitive attribute (Lemma~\ref{lem:dp-cov-ident}).
    \item We then decompose this dependence through the masked linear classification head to conclude the proof of Theorem~\ref{thm:dp-hsic}.
\end{enumerate}

\subsection{Connecting the Demographic Parity to HSIC}

\begin{lemma}[Multi-Group Covariance-Demographic Parity Bound]
\label{lem:dp-cov-ident}
For a categorical sensitive attribute $\sens \in \{1, \ldots, \nngroups\}$ with group prevalences $p_g = \mathbb{P}(\sens=g)$, the two smallest of which are $p_{(1)} \leq p_{(2)}$ with $p_{(1)} > 0$, and any measurable function $h: \xspace \to \real^K$ with $\expect[\|h(\vx)\|_2] < \infty$:
\begin{equation}
\label{eq:cG}
    \ddp(h) \leq c_\nngroups \sqrt{\HSIC(h(\vx), \vs)},
    \qquad
    c_\nngroups = \sqrt{\frac{(\nngroups-1)\big(p_{(1)}^{-2} + p_{(2)}^{-2}\big) + 2\,p_{(1)}^{-1}p_{(2)}^{-1}}{\nngroups}},
\end{equation}
where $\vs$ is the one-hot encoding of $\sens$.
\end{lemma}

\begin{proof}
Let $\meanemb_g := \expect[h(\vx) \mid \sens = g]$ and $\bar{\meanemb} := \expect[h(\vx)] = \sum_{g=1}^\nngroups p_g \meanemb_g$.

\textbf{Step 1. Compute covariance with each group indicator.}
For each indicator $s_g = \mathds{1}[\sens=g]$:
\begin{align}
    \mathrm{Cov}(h(\vx), s_g) &= \expect[h(\vx) \cdot s_g] - \expect[h(\vx)] \cdot \expect[s_g]
    = p_g \meanemb_g - p_g \bar{\meanemb} = p_g (\meanemb_g - \bar{\meanemb}).
\end{align}

\textbf{Step 2. Express HSIC.}
Under linear kernels (Definition~\ref{def:linear_hsic}):
\begin{equation}
    \HSIC(h(\vx), \vs) = \sum_{g=1}^\nngroups \|\mathrm{Cov}(h(\vx), s_g)\|_2^2 = \sum_{g=1}^\nngroups p_g^2 \|\meanemb_g - \bar{\meanemb}\|_2^2.
\end{equation}

\textbf{Step 3. Bound the gap between two groups.}
Fix two groups $g \neq g'$ and write $t = p_g\|\meanemb_g - \bar{\meanemb}\|_2$ and $t' = p_{g'}\|\meanemb_{g'} - \bar{\meanemb}\|_2$, so that $t^2$ and $t'^2$ are their terms in the sum of Step 2. By the triangle inequality,
\begin{equation}
    \|\meanemb_g - \meanemb_{g'}\|_2 \leq \|\meanemb_g - \bar{\meanemb}\|_2 + \|\meanemb_{g'} - \bar{\meanemb}\|_2 = \frac{t}{p_g} + \frac{t'}{p_{g'}}.
\end{equation}

\textbf{Step 4. Relate the two groups to the remaining groups.}
Since $\bar{\meanemb}$ is the weighted average of the group means, the weighted deviations $p_h(\meanemb_h - \bar{\meanemb})$ sum to zero over all groups. If two groups deviate from $\bar{\meanemb}$ in the same direction, the other groups must deviate in the opposite direction. Their deviations are also terms of the sum in Step 2, so they increase $\HSIC$. For $\nngroups = 2$, this forces $t = t'$, so $\HSIC(h(\vx), \vs) = 2t^2$. For $\nngroups \geq 3$, the other $\nngroups - 2$ groups must together offset the two, and by the Cauchy-Schwarz inequality and the reverse triangle inequality, $\|p_g(\meanemb_g - \bar{\meanemb}) + p_{g'}(\meanemb_{g'} - \bar{\meanemb})\|_2 \geq |t - t'|$, their terms in Step 2 add at least $(t - t')^2/(\nngroups - 2)$, so
\begin{equation}
    \HSIC(h(\vx), \vs) \geq t^2 + t'^2 + \frac{(t - t')^2}{\nngroups - 2}.
\end{equation}

\textbf{Step 5. Bound the demographic parity gap.}
By Step 3, the gap between groups $g$ and $g'$ is at most $t/p_g + t'/p_{g'}$, and by Step 4, $t$ and $t'$ cannot both be large for a given $\HSIC$. For $\nngroups \geq 3$, Step 4 constrains $(t, t')$ by a quadratic form, and maximizing the linear function $t/p_g + t'/p_{g'}$ under it is a direct application of the Cauchy-Schwarz inequality, using the inverse $\frac{1}{\nngroups}\bigl(\begin{smallmatrix} \nngroups-1 & 1 \\ 1 & \nngroups-1 \end{smallmatrix}\bigr)$ of the matrix of that form. For $\nngroups = 2$, Step 4 gives $t = t'$ directly. The largest possible value of $t/p_g + t'/p_{g'}$ is therefore reached when $t$ and $t'$ are proportional to
\begin{equation}
    \frac{\nngroups-1}{p_g} + \frac{1}{p_{g'}} \quad\text{and}\quad \frac{1}{p_g} + \frac{\nngroups-1}{p_{g'}},
\end{equation}
and equals $c_\nngroups\sqrt{\HSIC(h(\vx), \vs)}$ when $g, g'$ are the two rarest groups. Any other pair of groups gives no larger value, so
\begin{equation}
    \ddp(h) \leq c_\nngroups \sqrt{\HSIC(h(\vx), \vs)}.
\end{equation}
Finally, $c_\nngroups \leq \sqrt{2}/p_{(1)}$, since $2p_{(1)}^{-1}p_{(2)}^{-1} \leq p_{(1)}^{-2} + p_{(2)}^{-2}$ and $p_{(2)}^{-1} \leq p_{(1)}^{-1}$.
\end{proof}

\subsection{Proof of the Main Theorem}

\begin{proof}[Proof of Theorem~\ref{thm:dp-hsic}]
Throughout, the model computes $\model(\vx) = \mat{W}\zfair$, where $\zfair = \vb \odot \encoder(\vx)$ is the Fair Space and $\mat{W} \in \real^{K \times m}$ is the weight matrix of the linear classification head.

\textbf{Step 1. Apply Lemma~\ref{lem:dp-cov-ident}.}
Since $\norm{\mat{W}\zfair}_2 \leq \norm{\mat{W}}_F \norm{\zfair}_2$, we have $\expect[\norm{\model(\vx)}_2] \leq \norm{\mat{W}}_F \, \expect[\norm{\zfair}_2] < \infty$, so Lemma~\ref{lem:dp-cov-ident} applies to $\model$:
\begin{equation}
    \ddp(\model) \leq c_\nngroups \sqrt{\HSIC(\model(\vx), \vs)}.
\end{equation}

\textbf{Step 2. Decompose through the linear classification head.}
Since $\model(\vx) = \mat{W}\zfair$ is linear in $\zfair$:
\begin{equation}
    \mathrm{Cov}(\model(\vx), s_g) = \mat{W} \, \mathrm{Cov}(\zfair, s_g).
\end{equation}

Let $\mathcal{A} = \{j : b_j = 1\}$ with $|\mathcal{A}| = \dfair$, let $\vz_{\mathcal{A}}$ denote $\zfair$ restricted to the active dimensions, and let $\mat{W}_{\mathcal{A}} \in \real^{K \times \dfair}$ be the submatrix of $\mat{W}$ with the columns in $\mathcal{A}$. Since the inactive dimensions of $\zfair$ are identically zero, their covariances with $s_g$ vanish, so $\mat{W}\,\mathrm{Cov}(\zfair, s_g) = \mat{W}_{\mathcal{A}}\,\mathrm{Cov}(\vz_{\mathcal{A}}, s_g)$. Collecting these columns into $\mat{C}_{\mathcal{A}} \in \real^{\dfair \times \nngroups}$ and summing over $g$ (Definition~\ref{def:linear_hsic}):
\begin{equation}
    \HSIC(\model(\vx), \vs) = \|\mat{W}_{\mathcal{A}} \mat{C}_{\mathcal{A}}\|_F^2.
\end{equation}

\textbf{Step 3. Apply norm bounds.}
Using the submultiplicativity of the Frobenius norm\footnote{The inequality $\|\mat{A}\mat{B}\|_F \leq \|\mat{A}\|_F \|\mat{B}\|_F$ follows from applying Cauchy-Schwarz to each entry of the product.}:
\begin{equation}
    \|\mat{W}_{\mathcal{A}} \mat{C}_{\mathcal{A}}\|_F \leq \|\mat{W}_{\mathcal{A}}\|_F \cdot \|\mat{C}_{\mathcal{A}}\|_F.
\end{equation}

We bound the two terms separately:
\begin{enumerate}
    \item \textbf{Weight Matrix:} Since $|W_{ij}| \leq B$ for all entries (as $\|\mat{W}\|_{\max} \leq B$) and $\mat{W}_{\mathcal{A}}$ has dimensions $K \times \dfair$:
    \begin{equation}
        \|\mat{W}_{\mathcal{A}}\|_F = \sqrt{\sum_{i=1}^K \sum_{j \in \mathcal{A}} W_{ij}^2} \leq \sqrt{K \cdot \dfair \cdot B^2} = B\sqrt{K\dfair}.
    \end{equation}

    \item \textbf{Covariance Matrix:} Since the inactive dimensions of $\zfair$ contribute zero covariance, $\|\mat{C}_{\mathcal{A}}\|_F^2$ equals the full sum of squared covariances in Definition~\ref{def:linear_hsic}. Therefore:
    \begin{equation}
        \|\mat{C}_{\mathcal{A}}\|_F = \sqrt{\HSIC(\zfair, \vs)}.
    \end{equation}
\end{enumerate}

Combining these bounds with Step 2 yields:
\begin{equation}
    \sqrt{\HSIC(\model(\vx), \vs)} \leq B\sqrt{K\dfair} \cdot \sqrt{\HSIC(\zfair, \vs)}.
\end{equation}
Substituting this into the inequality from Step 1 completes the proof.
\end{proof}

\pagebreak

\section{Algorithm Pseudocode}
\label{app:algorithm}

\Cref{alg:fairsplit} gives the full pseudocode of \algname{} (Section~\ref{sec:methodology}). For each minibatch, it alternates two steps: (i) a gradient step on the model parameters $\allparams$, with moving-average updates of the centroids, and (ii) a closed-form update of the mask (Proposition~\ref{prop:mask_solution}).

\begin{algorithm}[h]
\caption{\algname{}}
\label{alg:fairsplit}
\begin{algorithmic}[1]
\State \textbf{Input:} Dataset $\mathcal{D}$, hyperparameters $\lamf, \lams, \rhof, \rhos$, moving-average rate $\alpha$, learning rate $\eta$, number of iterations $T$
\State \textbf{Initialize:} Model parameters $\allparams^{(0)}$, global centroids $\cclass$ (class) and $\cgroup$ (group), mask $\vb^{(0)}$ via Equation~\ref{eq:mask_solution}, binarized at $0.5$

\For{iteration $t = 1$ to $T$}
    \State Sample minibatch $\mathcal{B} = \{(\vx_i, \y_i, \sens_i)\} \subset \mathcal{D}$

    \Statex \textcolor{gray}{// --- Step 1: Representation update ---}
    \State Compute embeddings $\vz_i = \encoder(\vx_i)$ for $i \in \mathcal{B}$
    \State Compute batch means: $\tilde{\cclass}_k$ (class), $\tilde{\cgroup}_g$ (group)
    \State Update global centroids $\cclass_k, \cgroup_g$ with batch means
    \State Compute loss $\losstotal(\mathcal{B}; \allparams, \vb)$
    \State Update $\allparams \leftarrow \allparams - \eta \nabla_{\allparams} \losstotal$

    \Statex \textcolor{gray}{// --- Step 2: Mask update (fix $\allparams$) ---}
    \State Detach the embeddings $\vz_i$ (stop gradients)
    \For{dimension $j = 1$ to $m$}
        \State $V_j^{(\vartagClass)} \leftarrow \sum_{k} \sum_{i \in \mathcal{I}_k} (\symbZ_{ij} - \cclass_{kj})^2$
        \State $V_j^{(\vartagGroup)} \leftarrow \sum_{g} \sum_{i \in \mathcal{I}_g} (\symbZ_{ij} - \cgroup_{gj})^2$

        \State Compute closed-form mask value:
        \State $b_j^* \leftarrow \frac{V_j^{(\vartagGroup)}}{V_j^{(\vartagClass)} + V_j^{(\vartagGroup)}}$

        \State Update mask:
        \State Smooth $b_j^*$ with a moving average of rate $\alpha$ and binarize it at $0.5$ to obtain $b_j$
    \EndFor
\EndFor
\State \textbf{Return:} $\allparams$, $\vb$
\end{algorithmic}
\end{algorithm}

\section{Reproducibility and Experimental Details}
\label{app:reproducibility}
\subsection{Datasets}
\label{app:datasets}

We evaluate on five common fairness datasets: Adult and COMPAS as tabular settings, CelebA-A and UTKFace-AR as image settings, and FairFace for the ViT-B/16 experiment. \Cref{tab:datasets} summarizes all five datasets. The COMPAS sensitive attribute is intersectional (gender~$\times$~age), yielding four groups with $p_{\min}\!=\!0.05$, and UTKFace-AR uses 3-class age prediction with 3-group race as the sensitive attribute. For the experiments with large pretrained models (\Cref{sec:pretrained}), we use FairFace at $224 \times 224$ with a fine-tuned ViT-B/16, and CelebA at $224 \times 224$ with frozen foundation models.

\begin{table}[t]
\caption{Dataset statistics. $K$ and $\nngroups$ denote the number of target classes and sensitive groups, respectively. Target Dist.\ shows the class distribution, and Sens.\ Dist.\ shows the sensitive group distribution. FairFace is used in \Cref{sec:pretrained}, where each seed samples $60{,}000$ training and $10{,}000$ test images. Groups are ordered Female/Male (Adult, CelebA-A), White/Black/Other (UTKFace-AR), Male/Female (FairFace), and Male-Older/Male-Young/Female-Older/Female-Young (COMPAS).}
\label{tab:datasets}
\centering
\resizebox{\textwidth}{!}{%
\begin{sc}
\begin{tabular}{lllrcccc}
\toprule
Name & Modality & Target & Samples & $K$ & $\nngroups$ & Target Dist.\ (\%) & Sens.\ Dist.\ (\%) \\
\midrule
Adult & Tabular & Income & 45,222 & 2 & 2 & 75.2 / 24.8 & 32.5 / 67.5 \\
COMPAS & Tabular & 2-year recidivism & 5,278 & 2 & 4 & 53.0 / 47.0 & 58.8 / 21.7 / 14.6 / 5.0 \\
CelebA-A & Image & Attractive & 202,599 & 2 & 2 & 48.7 / 51.3 & 58.3 / 41.7 \\
UTKFace-AR & Image & Age (3-class) & 23,705 & 3 & 3 & 32.2 / 44.1 / 23.7 & 42.5 / 19.1 / 38.4 \\
\midrule
FairFace & Image & Age $\geq 30$ & 97,698 & 2 & 2 & 54.2 / 45.8 & 53.0 / 47.0 \\
\bottomrule
\end{tabular}
\end{sc}
}
\end{table}

\textbf{Adult}~\citep{dataset_adult,adult_2} is a tabular dataset where the task is to predict whether an individual's annual income exceeds \$50K. The sensitive attribute is gender (Female/Male). The dataset is hosted by the UCI Machine Learning Repository\footnote{\url{https://archive.ics.uci.edu/dataset/2/adult}} and is distributed under the Creative Commons Attribution 4.0 International (CC BY 4.0) license.

\textbf{COMPAS}~\citep{dataset_compas} is a tabular dataset where the task is to predict whether a defendant will reoffend within two years (binary recidivism). We apply the standard ProPublica preprocessing and restrict the data to African-American and Caucasian defendants. The sensitive attribute is intersectional, combining gender (Female/Male) with a binarized age group (Young: $\leq 25$, Older: $> 25$), yielding four groups with minimum prevalence $p_{\min}\!=\!0.05$. The data is released by ProPublica\footnote{Source: \url{https://github.com/propublica/compas-analysis}. Terms: \url{https://projects.propublica.org/datastore/terms/}.} under the ProPublica Data Store terms of use.

\textbf{CelebA-A}~\citep{dataset_celeba} consists of celebrity face images where we predict the ``Attractive'' attribute. The sensitive attribute is gender (Female/Male). Images are resized to $48 \times 48$ RGB. For the frozen foundation-model experiments (Appendix~\ref{app:scaling}), we instead use CelebA at $224 \times 224$ resolution, learn the Fair Space for predicting high cheekbones, and transfer it to predicting age (the \texttt{Young} attribute), with gender again as the sensitive attribute. We use the official split and hold out $10\%$ of the training images for validation ($146{,}493$ / $16{,}277$ / $19{,}962$ train / validation / test images). For CLIP ViT-L/14 and CLIP ResNet-50~\citep{radford2021clip}, we use the precomputed image features released by~\citet{dehdashtian2024fairerclip}, and for DINOv2 ViT-L/14~\citep{oquab2023dinov2} we extract features from the aligned images. Per the agreement on the official project page\footnote{\url{https://mmlab.ie.cuhk.edu.hk/projects/CelebA.html}}, the dataset is available for non-commercial research purposes only.

\textbf{UTKFace-AR}~\citep{dataset_utkface} contains face images with age and ethnicity annotations. We bin age into three groups: Young ($\leq 25$), Middle-aged ($26$-$45$), and Old ($> 45$), which yields reasonably balanced classes (32\%, 44\%, 24\%). The sensitive attribute is race, grouped into three categories: White, Black, and Other (Asian, Indian, and Others combined). Images are $48 \times 48$ grayscale, replicated to three channels. Per the terms on the official project page\footnote{\url{https://susanqq.github.io/UTKFace/}}, the dataset is available for non-commercial research purposes only.

\textbf{FairFace}~\citep{karkkainen2021fairface} is a face image dataset with a balanced race composition. We use it for the end-to-end ViT-B/16 experiment (Appendix~\ref{app:scaling}) at its native $224 \times 224$ resolution, predicting whether age is at least $30$, with gender as the sensitive attribute. For each seed, we sample $60{,}000$ training images from the official training split and $10{,}000$ test images from the official validation split. The dataset is released under the CC BY 4.0 license.

\subsection{Metrics}
\label{app:metrics}

Table~\ref{tab:metrics} summarizes the evaluation metrics used in our experiments. All values are reported in percent. The hard prediction $\hat{y}$ is obtained by thresholding the sigmoid output at $0.5$ for binary classification and by taking the argmax of the softmax output for multiclass. The threshold-free metrics (AP, ABCC, MCDP) operate directly on predicted probabilities. For datasets with a binary target and more than two groups (COMPAS), we use the variants from the right column with the $k$ index dropped, so each fairness metric becomes the maximum disparity across all pairs of sensitive groups.

\begin{table}[t]
\caption{Performance and fairness metrics. $\hat{y} \in \{1,\ldots,K\}$ is the predicted class, $y$ the true class, and $\sens \in \{1,\ldots,\nngroups\}$ the sensitive attribute, with class $2$ as the positive class for binary targets. $P$ and $R$ are precision and recall, and $P_r, R_r$ their values at the $r$-th threshold. With $\mathbb{P}_g(\cdot) = \mathbb{P}(\cdot \mid \sens=g)$, $\mathrm{TPR}_g^{(k)} = \mathbb{P}_g(\hat{y}{=}k \mid y{=}k)$, $\mathrm{FPR}_g^{(k)} = \mathbb{P}_g(\hat{y}{=}k \mid y{\neq}k)$, and $F_g^{(k)}$ the CDF of the predicted probability of class $k$ in group $g$, $\Delta(\cdot)_{g,g'}$ denotes the difference between groups $g$ and $g'$. Binary metrics drop the indices.}
\label{tab:metrics}
\centering
\small
\setlength{\tabcolsep}{4pt}
\renewcommand{\arraystretch}{1.15}
\begin{tabular}{@{}llll@{}}
\toprule
\textbf{Metric} & \textbf{Name} & \textbf{Binary} & \textbf{Multiclass / Multi-group} \\
\midrule
AP $\uparrow$ & Average Precision & $\sum_{r} (R_r - R_{r-1}) P_r$ & -- \\
F1 $\uparrow$ & F1 Score & $2PR/(P + R)$ & $\frac{1}{K}\sum_{k} 2P_kR_k/(P_k + R_k)$ \\
\midrule
$\Delta$DP $\downarrow$ & Demographic Parity & $|\Delta\mathbb{P}(\hat{y}{=}2)|$ & $\max_{k,g,g'} |\Delta\mathbb{P}_{g,g'}(\hat{y}{=}k)|$ \\
$\Delta$EOpp $\downarrow$ & Equal Opportunity & $|\Delta\mathrm{TPR}|$ & $\max_{k,g,g'} |\Delta\mathrm{TPR}_{g,g'}^{(k)}|$ \\
$\Delta$EOdd $\downarrow$ & Equalized Odds & $|\Delta\mathrm{TPR}| + |\Delta\mathrm{FPR}|$ & $\max_{k,g,g'} \big(|\Delta\mathrm{TPR}_{g,g'}^{(k)}| + |\Delta\mathrm{FPR}_{g,g'}^{(k)}|\big)$ \\
ABCC $\downarrow$ & Area Between CDF Curves & $\int_0^1 |\Delta F(t)| \,\mathrm{d}t$ & $\max_{k,g,g'} \int_0^1 |\Delta F_{g,g'}^{(k)}(t)| \,\mathrm{d}t$ \\
MCDP $\downarrow$ & Maximum CDF Disparity & $\max_t |\Delta F(t)|$ & $\max_{k,g,g',t} |\Delta F_{g,g'}^{(k)}(t)|$ \\
\bottomrule
\end{tabular}
\end{table}

\subsection{Implementation}
Our implementation builds upon the code of \citet{ffb}\footnote{\url{https://github.com/ahxt/fair_fairness_benchmark}} and \citet{DiffMCDP_MCDP_ICML}\footnote{\url{https://github.com/mitao-cat/icml24_mcdp}}, and adapts the mask and HSIC components from H-SPLID~\citep{H-SPLID_NeurIPS}\footnote{\url{https://github.com/neu-spiral/H-SPLID}}. We port MMD-B-Fair~\citep{mmd-b-fair} from its original repository\footnote{\url{https://github.com/namratadeka/mmd-b-fair}} into the FFB benchmark setup. Architecture details (MLP encoder for tabular settings, ResNet-18 backbone for image settings) and the linear classification head are described in \Cref{sec:hyperparameters}.

\subsection{Software and Hardware Setup}
\label{app:hardware}
Our implementation is written in Python using the PyTorch framework. The source code is publicly available at \codelink{}. All experiments were conducted on a single internal server running Ubuntu 22.04.3 LTS. The server is equipped with 64 CPUs, 512~GiB of system memory, and a single NVIDIA A100 GPU with 80~GB of VRAM.

\subsection{Hyperparameters}
\label{sec:hyperparameters}

For binary classification with binary sensitive attributes (Adult, CelebA-A), we follow the in-processing fair classification benchmark of~\citet{ffb}. For non-binary sensitive attributes (UTKFace-AR, COMPAS) and the multiclass target of UTKFace-AR, we make the following extensions. Baselines restricted to binary targets and sensitive attributes are excluded, namely FairMixup, DRAlign, DiffMCDP, PRemover, and MMD-B-Fair (see \Cref{tab:method_capability_summary}). For UTKFace-AR, performance is reported as F1 rather than AP, and val~F1 replaces val~AP in the $\geq 95\%$-of-ERM gate of the selection rule. Fairness metrics are extended to multiple classes and groups by taking the maximum disparity across all classes and all pairs of sensitive groups (formal definitions in \Cref{tab:metrics}).

For all experiments, we use the Adam optimizer~\cite{adam_optimizer}. In tabular experiments, we employ an MLP encoder with hidden layers $[512, 256, 64]$. For image experiments, we use a ResNet-18 backbone pretrained on ImageNet, where the final fully connected layer is modified to match the target dimensionality. Images are resized to $48 \times 48$, following the FFB benchmark~\citep{ffb}, and pixel values are scaled to $[0, 1]$. For the frozen backbones, the encoder is a projection on top of the precomputed features, an MLP with hidden layers $[1024, 512, 256]$ and no dropout, so the embedding has $m{=}256$ dimensions.~\looseness=-1

On Adult, COMPAS, CelebA-A, and UTKFace-AR, we use a learning rate of $1 \times 10^{-3}$ and a \texttt{StepLR} scheduler with step size $50$ and decay factor $\gamma=0.1$. Training is conducted for $150$ steps in total, following the FFB protocol~\citep{ffb}. In our preliminary experiments, we observed that the fairness regularizers worked better when added to a classifier that had already learned the task than when applied from the start. We therefore train every fairness method with the classification loss only for the first $50$ steps, after which its fairness objective is added and the optimizer and its learning-rate schedule are reset. ERM is trained without this warm-up. The $224 \times 224$ experiments (Appendix~\ref{app:scaling}) use longer schedules. The ViT-B/16 is fine-tuned for $300$ steps with a learning rate of $10^{-4}$, decayed every $100$ steps, and a batch size of $128$. The projection on top of the frozen backbones is trained for $600$ steps with a learning rate of $3 \times 10^{-4}$, decayed every $200$ steps, and a batch size of $1024$. In both cases, \algname{} also starts with $50$ steps of the classification loss only.

\begin{table}[h]
\centering
\caption{Training hyperparameters shared by all methods. The mask moving-average rate applies to \algname{} only.}
\label{tab:app_hyperparams}
\small
\begin{tabular}{@{}lccc@{}}
\toprule
\textbf{Hyperparameter} & \textbf{Benchmarks} & \textbf{ViT-B/16} & \textbf{Frozen backbones} \\
\midrule
Optimizer & Adam & Adam & Adam \\
Learning rate & $10^{-3}$ & $10^{-4}$ & $3 \times 10^{-4}$ \\
StepLR step size ($\gamma{=}0.1$) & 50 & 100 & 200 \\
Training steps & 150 & 300 & 600 \\
Warm-up steps (fairness methods) & 50 & 50 & 50 \\
Batch size & 1024 (tabular), 128 (image) & 128 & 1024 \\
Mask moving-average rate $\alpha$ & 0.99 & 0.99 & 0.99 \\
\bottomrule
\end{tabular}
\end{table}

For each fairness method, we run a grid search tuning all of its method-specific hyperparameters, with per-dataset settings differing only in encoder, batch size, and target and sensitive attributes. The selected configuration per (method, dataset) is in \Cref{tab:best-baselines} for baselines and in \Cref{tab:fairsplit-params} for \algname{}.~\looseness=-1

\begin{table}[h]
\caption{Selected best hyperparameters per baseline method and dataset. Among the configurations whose validation AP (validation F1 for UTKFace-AR) reaches at least $95\%$ of ERM's, we select the one with the lowest mean validation $\Delta$DP (see \Cref{sec:experiments}). Cells marked ``---'' indicate the method is not applicable to the dataset's setting.}
\label{tab:best-baselines}
\centering
\small
\begin{tabular}{@{}llcccc@{}}
\toprule
\textbf{Method} & \textbf{Parameter} & \textbf{Adult} & \textbf{CelebA-A} & \textbf{UTKFace-AR} & \textbf{COMPAS} \\
\midrule
AdvDebias & $\lambda$ & $2$ & $0.2$ & $0.35$ & $1$ \\
\midrule
FairMixup & $\lambda$ & $0.8$ & $1$ & --- & --- \\
\midrule
DRAlign & $\lambda$ & $0.5$ & $0.2$ & --- & --- \\
\quad & $\beta$ & $0.01$ & $0.01$ & --- & --- \\
\midrule
HSIC$(Y,S)$ & $\lambda$ & $70$ & $50$ & $300$ & $200$ \\
\quad & $\sigma$ & $1$ & $5$ & $5$ & $1$ \\
\midrule
DiffMCDP & $\lambda$ & $0.5$ & $0.3$ & --- & --- \\
\quad & $\tau$ & $5$ & $20$ & --- & --- \\
\midrule
PRemover & $\lambda$ & $0.005$ & $0.025$ & --- & --- \\
\midrule
MMD-B-Fair & $\lambda$ & $0.5$ & $0.2$ & --- & --- \\
\quad & $\sigma$ range & $(0.5,2.0)$ & $(0.7,3.0)$ & --- & --- \\
\bottomrule
\end{tabular}
\end{table}

\begin{table}[h]
\caption{Selected \algname{} hyperparameters per dataset, with the same selection rule as in \Cref{tab:best-baselines}. $\sigma$ is fixed at the modality default (tabular vs.\ image). For the frozen backbones, the hyperparameters are selected on CLIP ViT-L/14 and reused unchanged on the other two backbones.}
\label{tab:fairsplit-params}
\centering
\small
\begin{tabular}{@{}lcccccc@{}}
\toprule
\textbf{Parameter} & \textbf{Adult} & \textbf{CelebA-A} & \textbf{UTKFace-AR} & \textbf{COMPAS} & \textbf{ViT-B/16} & \textbf{Frozen} \\
\midrule
$\lamf$ & $0.025$ & $0.01$ & $0.1$ & $0.01$ & $0.01$ & $0.01$ \\
$\lams$ & $0.05$ & $0.025$ & $0.05$ & $0.2$ & $0.025$ & $0.05$ \\
$\rhof$ & $0.4$ & $0.5$ & $1.3$ & $0.3$ & $1.5$ & $1.0$ \\
$\rhos$ & $0.005$ & $0.5$ & $0.1$ & $0.005$ & $2.0$ & $0.5$ \\
$\sigma$ & $2$ & $5$ & $5$ & $2$ & $5$ & $5$ \\
\bottomrule
\end{tabular}
\end{table}

\section{Additional Results}
\label{app:additional_results}

\subsection{Bound Tightness}
\label{app:linear_bound_validation}

We evaluate the bound of Theorem~\ref{thm:dp-hsic} on trained models. For each run, we compute the right-hand side of Equation~\ref{eq:dp_bound} on the test set, with $B$ the largest absolute entry of the trained classification head and $\HSIC$ the linear-kernel HSIC of Definition~\ref{def:linear_hsic}, and compare it with the demographic parity gap $\ddp(\model)$ of the logits. \Cref{fig:bound_tightness} reports both on Adult, CelebA-A, UTKFace-AR, and COMPAS as $\rhof$ varies, together with the demographic parity gap of the predicted classes. The bound follows the gap as $\rhof$ changes. At the selected $\rhof$, it exceeds the gap by a factor of $2.0$ on Adult, $2.4$ on CelebA-A, $5.5$ on UTKFace-AR, and $14.3$ on COMPAS. The factor is largest on COMPAS, whose smallest group has a prevalence of $0.05$, since $c_\nngroups$ grows with $1/p_{\min}$.

\begin{figure}[h]
    \centering
    \includegraphics[width=\linewidth]{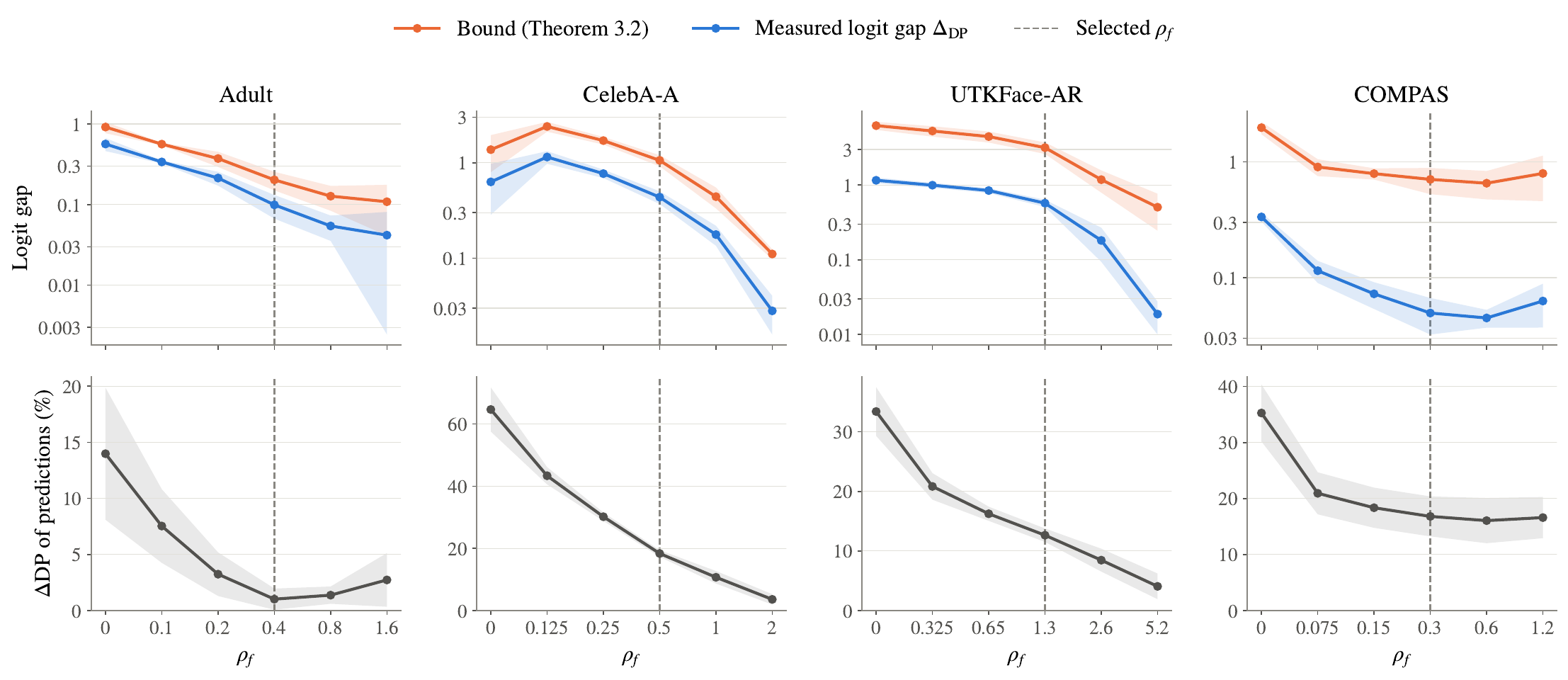}
    \caption{Bound of Theorem~\ref{thm:dp-hsic} and the demographic parity gap of the logits (top), and the demographic parity gap of the predicted classes (bottom), as the HSIC regularization strength $\rhof$ varies. Shaded regions are the mean $\pm$ std over 5 seeds, and the dashed line marks the reported $\rhof$.}
    \label{fig:bound_tightness}
\end{figure}

\subsection{Robustness Across Encoders and Embedding Dimensionalities}
\label{app:encoder_robustness}

We evaluate \algname{} across five pretrained encoders spanning embedding dimensionalities $m \in \{512, 960, 1280, 2048\}$. \algname{} uses the same hyperparameters for every encoder, and only the encoder is changed, with no per-architecture retuning. This setup tests two properties at once: (i) whether \algname{} is robust to the choice of encoder, and (ii) whether the closed-form mask (Proposition~\ref{prop:mask_solution}) adapts $\dfair$ to embedding spaces of different sizes. Results on UTKFace-AR and CelebA-A are reported in \Cref{tab:encoder_robustness}, averaged over 5 seeds. We report F1$\uparrow$ (UTKFace-AR) or AP$\uparrow$ (CelebA-A) as the utility metric, $\Delta$DP$\downarrow$ as the fairness metric, the relative $\Delta$DP reduction over ERM, and the learned Fair Space dimensionality $\dfair$ together with the encoder's embedding dimensionality $m$.~\looseness=-1

Across all five backbones, \algname{} reduces the demographic parity gap by between $45\%$ and $56\%$ on UTKFace-AR and between $38\%$ and $65\%$ on CelebA-A, at a utility cost that varies by encoder. The mask adapts $\dfair$ to each embedding without manual thresholding or per-encoder tuning. The active fraction $\dfair / m$ varies across backbones, for example $233/512 \approx 46\%$ for ResNet-18 versus $1161/1280 \approx 91\%$ for EfficientNetV2-S. Since the bound of Theorem~\ref{thm:dp-hsic} scales with $\sqrt{\dfair}$, the mask balances fairness and utility in both low- and high-dimensional embedding spaces. \algname{} therefore transfers across these five encoders and their embedding sizes without retuning.

\begin{table}[t]
\centering
\caption{Encoder robustness on UTKFace-AR (3-class age, 3 racial groups) and CelebA-A. \algname{} uses the same hyperparameters for every encoder, and only the encoder is changed, with no per-architecture retuning. Since the hyperparameters are not tuned per encoder, the $95\%$ utility gate of the main tables is not applied here. Utility and $\Delta$DP columns show ERM $\to$ \algname{}, and Red.\ is the relative reduction of $\Delta$DP. Mean over 5 seeds.~\looseness=-1}
\label{tab:encoder_robustness}
\small
\setlength{\tabcolsep}{3pt}
\resizebox{\linewidth}{!}{
\begin{tabular}{l c cccc cccc}
\toprule
& & \multicolumn{4}{c}{UTKFace-AR} & \multicolumn{4}{c}{CelebA-A} \\
\cmidrule(lr){3-6} \cmidrule(lr){7-10}
Encoder & $m$ & F1$\uparrow$ & $\Delta$DP$\downarrow$ & Red. & $\dfair$ & AP$\uparrow$ & $\Delta$DP$\downarrow$ & Red. & $\dfair$ \\
\midrule
ResNet-18        & 512  & 73.8 $\to$ 70.4 & 27.9 $\to$ 12.6 & 55\% & 233  & 84.5 $\to$ 80.4 & 52.8 $\to$ 18.3 & 65\% & 334  \\
ResNet-34        & 512  & 73.8 $\to$ 69.1 & 28.9 $\to$ 12.8 & 56\% & 243  & 84.5 $\to$ 78.1 & 53.2 $\to$ 18.4 & 65\% & 297  \\
ResNet-50        & 2048 & 74.8 $\to$ 72.2 & 27.9 $\to$ 14.9 & 47\% & 1003 & 85.9 $\to$ 80.6 & 52.4 $\to$ 25.5 & 51\% & 1230 \\
EfficientNetV2-S & 1280 & 73.2 $\to$ 73.8 & 29.1 $\to$ 15.9 & 45\% & 1161 & 84.9 $\to$ 82.6 & 50.6 $\to$ 31.2 & 38\% & 777  \\
MobileNetV4-S    & 960  & 70.4 $\to$ 67.5 & 28.5 $\to$ 13.9 & 51\% & 584  & 84.0 $\to$ 74.7 & 51.6 $\to$ 19.3 & 63\% & 419  \\
\bottomrule
\end{tabular}
}
\end{table}

\paragraph{Embedding size.} To test whether the embedding size alone affects fairness, we vary the dimensionality $m$ of the encoder's output embedding on Adult (MLP layers $[512, 256, m]$) over $m \in \{32, 64, 128, 256, 512\}$, with 5 seeds per setting (\Cref{tab:embedding_size}). Unlike in the main tables, the $95\%$ utility selection rule is not applied, so the utility cost varies across sizes. ERM barely changes across the whole range, from $19.8$ to $17.5$ in $\Delta$DP with ABCC staying flat, even at $m = 32$, which is smaller than the Fair Space that \algname{} learns at $m = 64$ ($\dfair \approx 47$). The embedding size is therefore not by itself a fairness mechanism. \algname{} keeps $\Delta$DP between $1.0$ and $3.3$ at every size with no systematic trend, and reduces ABCC from about $18.7$ to between $0.8$ and $4.3$, at a utility cost of $3.2$ to $6.5$ AP. The fairness gain therefore comes from the decomposition with HSIC-based regularization rather than from the embedding size.

\begin{table}[t]
\centering
\caption{Embedding-size ablation on Adult (5 seeds).}
\label{tab:embedding_size}
\small
\setlength{\tabcolsep}{4pt}
\begin{tabular}{l l ccccc}
\toprule
& & $m=32$ & $m=64$ & $m=128$ & $m=256$ & $m=512$ \\
\midrule
\multirow{2}{*}{$\Delta$DP $\downarrow$} & ERM & 19.8$\pm$0.8 & 19.4$\pm$0.6 & 18.5$\pm$0.5 & 18.2$\pm$0.3 & 17.5$\pm$0.2 \\
 & \algname{} & 3.3$\pm$1.2 & 1.0$\pm$0.9 & 3.2$\pm$2.6 & 1.3$\pm$0.5 & 2.5$\pm$2.4 \\
\midrule
\multirow{2}{*}{AP $\uparrow$} & ERM & 76.8$\pm$0.5 & 76.7$\pm$0.5 & 77.0$\pm$0.5 & 77.3$\pm$0.4 & 77.5$\pm$0.4 \\
 & \algname{} & 71.5$\pm$0.8 & 72.9$\pm$0.9 & 71.6$\pm$1.2 & 70.8$\pm$1.2 & 74.3$\pm$1.0 \\
\midrule
\multirow{2}{*}{ABCC $\downarrow$} & ERM & 18.8$\pm$0.3 & 18.7$\pm$0.2 & 18.5$\pm$0.2 & 18.6$\pm$0.2 & 18.5$\pm$0.1 \\
 & \algname{} & 0.9$\pm$0.2 & 2.4$\pm$0.7 & 1.5$\pm$0.3 & 0.8$\pm$0.2 & 4.3$\pm$0.5 \\
\bottomrule
\end{tabular}
\end{table}

\subsection{Fine-Tuned Vision Transformers and Frozen Foundation Models}
\label{app:scaling}
\textbf{ViT-B/16.} We fine-tune an ImageNet-pretrained ViT-B/16 end-to-end as the \algname{} encoder on FairFace at its native $224 \times 224$ resolution, predicting whether age is at least $30$ with gender as the sensitive attribute. ERM uses the same architecture and optimization. \algname{} lowers $\Delta$DP from $23.0$ to $9.5$ and more than halves every other fairness metric while retaining $97\%$ of ERM's AP, and the mask keeps $552$ of the $768$ dimensions (\Cref{tab:pretrained}).

\textbf{Frozen foundation models and transfer to a related task.} We apply \algname{} on top of frozen CLIP ViT-L/14, CLIP ResNet-50, and DINOv2 ViT-L/14 features of CelebA, with gender as the sensitive attribute. Since the backbone stays frozen, the split acts on a small trainable projection ($1024$-$512$-$256$) trained to predict high cheekbones (Task A). We then keep the projection and the mask fixed and train only a new linear head on $\zfair$ to predict whether a person is young (Task B), without using the sensitive attribute. The hyperparameters selected on CLIP ViT-L/14 are reused unchanged on the other two backbones. ERM uses the same architecture and optimization and follows the same transfer protocol, with its head reading the full embedding. On Task B, which the Fair Space was never trained or tuned on, \algname{} improves every fairness metric on all three backbones while retaining at least $95\%$ of ERM's AP (\Cref{tab:transfer}).

\begin{table}[h]
\centering
\caption{Transfer of the Fair Space from Task A (\texttt{High\_Cheekbones}) to Task B (\texttt{Young}) on three frozen backbones (5 seeds). Task B uses a fresh linear head on the frozen representation.}
\label{tab:transfer}
\setlength{\tabcolsep}{3pt}
\resizebox{\linewidth}{!}{
\begin{tabular}{l l l cccccc}
\toprule
Backbone & Task & Method & AP $\uparrow$ & $\Delta$DP $\downarrow$ & ABCC $\downarrow$ & $\Delta$EOpp $\downarrow$ & $\Delta$EOdd $\downarrow$ & $\dfair / m$ \\
\midrule
\multirow{4}{*}{CLIP ViT-L/14} & \multirow{2}{*}{A} & ERM & \textbf{93.7$\pm$0.0} & 27.8$\pm$0.4 & 26.0$\pm$0.2 & 15.8$\pm$0.4 & 18.3$\pm$1.0 & 256 / 256 \\
 & & \cellcolor{gray!15}\algname{} & \cellcolor{gray!15}89.3$\pm$0.4 & \cellcolor{gray!15}\textbf{13.1$\pm$1.2} & \cellcolor{gray!15}\textbf{9.0$\pm$0.6} & \cellcolor{gray!15}\textbf{1.2$\pm$0.7} & \cellcolor{gray!15}\textbf{12.1$\pm$1.0} & \cellcolor{gray!15}174 / 256 \\
 & \multirow{2}{*}{B} & ERM & \textbf{96.5$\pm$0.1} & 21.4$\pm$0.9 & 20.0$\pm$0.6 & 4.9$\pm$0.5 & 23.6$\pm$2.4 & 256 / 256 \\
 & & \cellcolor{gray!15}\algname{} & \cellcolor{gray!15}93.1$\pm$0.4 & \cellcolor{gray!15}\textbf{9.1$\pm$1.3} & \cellcolor{gray!15}\textbf{6.4$\pm$0.7} & \cellcolor{gray!15}\textbf{0.8$\pm$0.3} & \cellcolor{gray!15}\textbf{4.8$\pm$1.6} & \cellcolor{gray!15}174 / 256 \\
\midrule
\multirow{4}{*}{CLIP ResNet-50} & \multirow{2}{*}{A} & ERM & \textbf{94.5$\pm$0.0} & 27.1$\pm$0.2 & 26.3$\pm$0.1 & 14.0$\pm$0.5 & 15.7$\pm$0.5 & 256 / 256 \\
 & & \cellcolor{gray!15}\algname{} & \cellcolor{gray!15}90.8$\pm$0.7 & \cellcolor{gray!15}\textbf{15.4$\pm$1.3} & \cellcolor{gray!15}\textbf{10.5$\pm$0.9} & \cellcolor{gray!15}\textbf{1.5$\pm$1.2} & \cellcolor{gray!15}\textbf{10.9$\pm$0.6} & \cellcolor{gray!15}168 / 256 \\
 & \multirow{2}{*}{B} & ERM & \textbf{96.6$\pm$0.1} & 23.4$\pm$0.4 & 21.0$\pm$0.3 & 6.1$\pm$0.5 & 29.4$\pm$1.5 & 256 / 256 \\
 & & \cellcolor{gray!15}\algname{} & \cellcolor{gray!15}93.6$\pm$0.4 & \cellcolor{gray!15}\textbf{12.2$\pm$2.4} & \cellcolor{gray!15}\textbf{7.9$\pm$1.5} & \cellcolor{gray!15}\textbf{0.9$\pm$0.5} & \cellcolor{gray!15}\textbf{8.5$\pm$5.0} & \cellcolor{gray!15}168 / 256 \\
\midrule
\multirow{4}{*}{DINOv2 ViT-L/14} & \multirow{2}{*}{A} & ERM & \textbf{94.4$\pm$0.0} & 28.7$\pm$0.2 & 26.8$\pm$0.2 & 16.8$\pm$0.3 & 19.7$\pm$0.5 & 256 / 256 \\
 & & \cellcolor{gray!15}\algname{} & \cellcolor{gray!15}90.7$\pm$0.5 & \cellcolor{gray!15}\textbf{15.2$\pm$1.0} & \cellcolor{gray!15}\textbf{10.3$\pm$0.6} & \cellcolor{gray!15}\textbf{1.7$\pm$0.8} & \cellcolor{gray!15}\textbf{11.2$\pm$0.4} & \cellcolor{gray!15}166 / 256 \\
 & \multirow{2}{*}{B} & ERM & \textbf{95.2$\pm$0.1} & 23.5$\pm$0.4 & 22.1$\pm$0.2 & 7.3$\pm$0.3 & 35.3$\pm$1.7 & 256 / 256 \\
 & & \cellcolor{gray!15}\algname{} & \cellcolor{gray!15}90.9$\pm$0.5 & \cellcolor{gray!15}\textbf{9.1$\pm$1.8} & \cellcolor{gray!15}\textbf{5.9$\pm$1.2} & \cellcolor{gray!15}\textbf{0.5$\pm$0.4} & \cellcolor{gray!15}\textbf{7.1$\pm$1.7} & \cellcolor{gray!15}166 / 256 \\
\bottomrule
\end{tabular}
}
\end{table}

\subsection{Probing the Fair and Sensitive Spaces}
\label{app:probes}
To measure what a representation reveals about the sensitive attribute, we freeze a trained model and fit probes that predict $\sens$ from it. The linear probe is a logistic regression, and the MLP probe has two hidden layers of $256$ units. Both are fit on training embeddings and evaluated by balanced accuracy on test embeddings, averaged over $5$ seeds. \Cref{tab:probes} compares the full embedding, the Fair Space, and the Sensitive Space of each \algname{} model. \Cref{tab:probes_baselines} compares \algname{} with the baselines on the input of their classifiers.

\textbf{Fair and Sensitive Spaces.} \Cref{tab:probes} reports four tasks on frozen CLIP ViT-L/14, with the setup of Appendix~\ref{app:scaling}, and the CelebA-A model of \Cref{tab:combined_performance}. In every setting, the Sensitive Space predicts $\sens$ as well as the full embedding, while a linear probe predicts it less accurately from the Fair Space, at $50$-$64\%$ on frozen CLIP features. A non-linear MLP probe recovers $\sens$ from the Fair Space in all settings.\looseness=-1

\textbf{Comparison with baselines.} \Cref{tab:probes_baselines} probes the Fair Space of \algname{} and the full embedding of each baseline, which is the input of their classifiers. A linear probe predicts $\sens$ least accurately from the Fair Space of \algname{} on Adult, UTKFace-AR, and COMPAS, by $8.4$ to $18.8$ points below the next-best method, and at the level of random guessing on COMPAS. With the MLP probe, the Fair Space of \algname{} is also the least predictive on UTKFace-AR and COMPAS, while on Adult all methods lie within $2.1$ points.

\textbf{Scope.} These results are consistent with Theorem~\ref{thm:dp-hsic}, which bounds the demographic parity of a linear head through the linear dependence between $\zfair$ and $\sens$. \algname{} reduces this dependence rather than removing $\sens$ from the representation entirely. Complete removal would force the group-wise error rates of every classifier to sum to at least $\Delta\pi$~\citep{zhao2022inherent}.

\begin{table}[t]
\centering
\caption{Balanced accuracy (\%, 5 seeds) of probes predicting the sensitive attribute from the full embedding $\vz$, the Fair Space $\zfair$, and the Sensitive Space $\zsens$ of the same trained model, where random guessing scores $50\%$. CLIP ViT-L/14 is frozen (Appendix~\ref{app:scaling}), and ResNet-18 is the CelebA-A model of \Cref{tab:combined_performance}. For $\vz$, AP, and $\Delta$DP, arrows show ERM $\to$ \algname{} with the same architecture. Lower probe accuracy means that $\sens$ is harder to predict.}
\label{tab:probes}
\small
\setlength{\tabcolsep}{4pt}
\resizebox{\linewidth}{!}{
\begin{tabular}{l l c ccc c c c c}
\toprule
& & & \multicolumn{3}{c}{Linear probe} & MLP probe & & & \\
\cmidrule(lr){4-6} \cmidrule(lr){7-7}
Backbone & Task & $\Delta\pi$ & $\vz$ & $\zfair$ & $\zsens$ & $\zfair$ & $\dfair / m$ & AP $\uparrow$ & $\Delta$DP $\downarrow$ \\
\midrule
\multirow{4}{*}{CLIP ViT-L/14} & \texttt{Smiling}          & 14 & 98.1 $\to$ 99.0 & 50.4$\pm$0.3 & 99.0$\pm$0.0 & 95.9$\pm$1.1 & 167 / 256 & 97.9 $\to$ 97.0 & 15.2 $\to$ 8.5 \\
                               & \texttt{Young}            & 24 & 98.5 $\to$ 99.0 & 61.0$\pm$6.9 & 99.0$\pm$0.0 & 89.0$\pm$1.9 & 91 / 256  & 97.6 $\to$ 95.5 & 23.2 $\to$ 4.9 \\
                               & \texttt{High\_Cheekbones} & 25 & 98.4 $\to$ 99.0 & 54.7$\pm$2.1 & 99.0$\pm$0.0 & 97.2$\pm$0.9 & 174 / 256 & 93.7 $\to$ 89.3 & 27.8 $\to$ 13.1 \\
                               & \texttt{Wavy\_Hair}       & 31 & 98.8 $\to$ 99.0 & 64.4$\pm$4.5 & 99.0$\pm$0.0 & 94.8$\pm$1.1 & 115 / 256 & 85.6 $\to$ 83.3 & 35.4 $\to$ 26.6 \\
\midrule
ResNet-18 & \texttt{Attractive} & 40 & 85.7 $\to$ 94.8 & 88.2$\pm$2.4 & 94.8$\pm$0.2 & 94.5$\pm$0.3 & 334 / 512 & 84.5 $\to$ 80.4 & 52.8 $\to$ 18.3 \\
\bottomrule
\end{tabular}
}
\end{table}

\begin{table}[t]
\centering
\caption{Balanced accuracy (\%, 5 seeds) of probes predicting the sensitive attribute from the input of each method's classifier, the Fair Space $\zfair$ for \algname{} and the full embedding $\vz$ for the baselines. Random guessing scores $50\%$ on Adult, $33\%$ on UTKFace-AR, and $25\%$ on COMPAS. The probed models are those of \Cref{tab:combined_performance,tab:utkface_compas}. Lower probe accuracy means that $\sens$ is harder to predict, and the lowest value per column is in bold. A dash (---) marks baselines that support only binary targets and binary sensitive attributes, so they cannot be applied to UTKFace-AR (three age classes, three racial groups) or COMPAS (four groups).}
\label{tab:probes_baselines}
\small
\setlength{\tabcolsep}{4pt}
\begin{tabular}{l cc cc cc}
\toprule
& \multicolumn{2}{c}{Adult} & \multicolumn{2}{c}{UTKFace-AR} & \multicolumn{2}{c}{COMPAS} \\
\cmidrule(lr){2-3} \cmidrule(lr){4-5} \cmidrule(lr){6-7}
Method & Linear & MLP & Linear & MLP & Linear & MLP \\
\midrule
ERM         & 71.7$\pm$0.8 & 79.1$\pm$0.9 & 69.4$\pm$0.8 & 68.2$\pm$0.9 & 45.1$\pm$0.9 & 48.5$\pm$0.6 \\
AdvDebias   & 80.0$\pm$2.5 & 80.1$\pm$1.9 & 70.0$\pm$1.2 & 66.4$\pm$1.8 & 48.7$\pm$0.5 & 47.8$\pm$0.7 \\
FairMixup   & 75.9$\pm$0.6 & 80.6$\pm$0.7 & --- & --- & --- & --- \\
DRAlign     & 72.1$\pm$2.4 & 79.7$\pm$0.9 & --- & --- & --- & --- \\
HSIC$(Y,S)$ & 72.0$\pm$1.6 & 78.6$\pm$2.5 & 64.8$\pm$1.8 & 62.2$\pm$1.8 & 44.1$\pm$0.7 & 47.5$\pm$0.6 \\
DiffMCDP    & 73.7$\pm$1.4 & 79.5$\pm$2.5 & --- & --- & --- & --- \\
PRemover    & 71.0$\pm$2.1 & 79.3$\pm$2.5 & --- & --- & --- & --- \\
MMD-B-Fair  & 68.9$\pm$0.7 & 79.3$\pm$1.0 & --- & --- & --- & --- \\
\rowcolor{gray!15}\textbf{\algname{}} & \textbf{60.5$\pm$1.4} & \textbf{78.5$\pm$1.6} & \textbf{48.9$\pm$1.7} & \textbf{60.6$\pm$1.8} & \textbf{25.3$\pm$0.3} & \textbf{35.2$\pm$4.3} \\
\bottomrule
\end{tabular}
\end{table}

\subsection{Empirical Performance with Non-linear Classifier Heads}
\label{app:nonlinear_head}

Theorem~\ref{thm:dp-hsic} requires a linear classification head. We test how much this assumption matters in practice by replacing the linear head with a non-linear MLP head. \Cref{tab:nonlinear_head} compares the linear-head and non-linear-head settings on CelebA-A and UTKFace-AR.

\begin{table}[h]
\centering
\caption{\algname{} with linear and non-linear classification heads on $\zfair$ (5 seeds). The utility metric is AP on CelebA-A and F1 on UTKFace-AR. For each dataset and head, bold marks the better of ERM and \algname{}.}
\label{tab:nonlinear_head}
\small
\setlength{\tabcolsep}{4pt}
\resizebox{\linewidth}{!}{
\begin{tabular}{ll ccccccc}
\toprule
Method & Head & AP/F1$\uparrow$ & $\Delta$DP$\downarrow$ & ABCC$\downarrow$ & MCDP$\downarrow$ & $\Delta$EOpp$\downarrow$ & $\Delta$EOdd$\downarrow$ & $\dfair$ \\
\midrule
\multicolumn{9}{l}{\emph{CelebA-A}} \\
ERM       & Linear     & \textbf{84.5$\pm$0.4} & 52.8$\pm$0.8 & 35.9$\pm$0.8 & 53.4$\pm$0.8 & 36.8$\pm$1.5 & 75.0$\pm$2.1 & 512 \\
\rowcolor{gray!15}\cellcolor{white}\algname{} & Linear     & 80.4$\pm$1.1 & \textbf{18.3$\pm$1.1} & \textbf{9.4$\pm$1.3} & \textbf{19.6$\pm$1.8} & \textbf{2.6$\pm$0.3} & \textbf{5.2$\pm$1.7} & 334 \\
ERM       & Non-linear & \textbf{85.3$\pm$0.4} & 53.7$\pm$1.9 & 37.5$\pm$1.2 & 54.0$\pm$2.0 & 37.3$\pm$3.0 & 75.6$\pm$4.9 & 512 \\
\rowcolor{gray!15}\cellcolor{white}\algname{} & Non-linear & 81.7$\pm$0.5 & \textbf{23.1$\pm$3.6} & \textbf{13.5$\pm$1.5} & \textbf{24.3$\pm$3.3} & \textbf{4.4$\pm$2.2} & \textbf{11.0$\pm$5.0} & 369 \\
\midrule
\multicolumn{9}{l}{\emph{UTKFace-AR}} \\
ERM       & Linear     & \textbf{73.8$\pm$0.7} & 27.9$\pm$1.4 & 27.1$\pm$2.6 & 36.2$\pm$2.4 & 29.1$\pm$5.2 & 35.9$\pm$3.0 & 512 \\
\rowcolor{gray!15}\cellcolor{white}\algname{} & Linear     & 70.4$\pm$1.5 & \textbf{12.6$\pm$1.1} & \textbf{9.1$\pm$1.0} & \textbf{18.7$\pm$1.2} & \textbf{14.9$\pm$4.6} & \textbf{22.1$\pm$4.9} & 233 \\
ERM       & Non-linear & \textbf{74.9$\pm$1.0} & 30.4$\pm$1.7 & 28.8$\pm$1.7 & 39.9$\pm$2.9 & 33.2$\pm$5.4 & 40.8$\pm$4.8 & 512 \\
\rowcolor{gray!15}\cellcolor{white}\algname{} & Non-linear & 70.3$\pm$2.1 & \textbf{17.5$\pm$3.7} & \textbf{12.4$\pm$1.8} & \textbf{21.5$\pm$2.6} & \textbf{23.4$\pm$3.4} & \textbf{26.5$\pm$2.2} & 289 \\
\bottomrule
\end{tabular}
}
\end{table}

The fairness gains are smaller with the non-linear head, as expected, since a non-linear head can use sensitive information in $\zfair$ that a linear head cannot. They remain substantial: $\Delta$DP drops from $53.7$ to $23.1$ on CelebA-A ($57\%$ reduction) and from $30.4$ to $17.5$ on UTKFace-AR ($42\%$ reduction), with utility within about $5$ points of ERM in both settings. The mask keeps $\dfair = 369$ dimensions on CelebA-A and $289$ on UTKFace-AR. The linear head is therefore needed for the theoretical guarantee, but not for the method to improve fairness in practice.

\subsection{Performance under Imposed Group Imbalance}
\label{app:imbalance}

This subsection complements the intersectional COMPAS experiment in the main paper (\Cref{tab:utkface_compas}) with a controlled imbalance setting on a high-dimensional image dataset. We resample CelebA-A so that one sensitive group constitutes only $5\%$ of the training data. Each method keeps its hyperparameters from \Cref{tab:best-baselines,tab:fairsplit-params}, and only the training distribution changes.

\Cref{tab:imbalance_celeba} shows that \algname{} achieves the best score on all five fairness metrics while matching HSIC$(Y,S)$ on AP. AdvDebias stays close to ERM on every fairness metric under this imbalance.

\begin{table}[h]
\centering
\caption{CelebA-A with an imposed $95/5$ group imbalance in the training data (5 seeds). Each method keeps its hyperparameters from \Cref{tab:best-baselines,tab:fairsplit-params}. Best in \textbf{bold}, second-best \underline{underlined}.}
\label{tab:imbalance_celeba}
\small
\begin{tabular}{l ccccc|c}
\toprule
 & $\Delta$DP$\downarrow$ & ABCC$\downarrow$ & MCDP$\downarrow$ & $\Delta$EOpp$\downarrow$ & $\Delta$EOdd$\downarrow$ & AP$\uparrow$ \\
\midrule
ERM        & 54.0$\pm$0.4 & 34.2$\pm$0.8 & 54.4$\pm$0.5 & 42.4$\pm$1.8 & 79.3$\pm$1.6 & \underline{84.1$\pm$0.3} \\
\midrule
AdvDebias  & 52.3$\pm$2.7 & 35.1$\pm$1.1 & 53.6$\pm$1.9 & 35.4$\pm$8.8 & 73.3$\pm$7.0 & \textbf{85.4$\pm$0.3} \\
HSIC$(Y,S)$ & \underline{19.4$\pm$1.9} & \underline{10.2$\pm$1.1} & \underline{23.1$\pm$2.8} & \underline{2.3$\pm$1.2} & \underline{7.2$\pm$4.0} & 80.7$\pm$0.8 \\
\midrule
\rowcolor{gray!15}\textbf{\algname{}} & \textbf{18.2$\pm$1.6} & \textbf{9.7$\pm$1.0} & \textbf{22.8$\pm$2.8} & \textbf{1.8$\pm$1.3} & \textbf{7.1$\pm$3.5} & 80.7$\pm$0.7 \\
\bottomrule
\end{tabular}
\end{table}

\subsection{Fairness-Utility Trade-off and the Base-Rate Gap}
\label{app:basegap}
To isolate the effect of the base-rate gap, we fix the task (\texttt{Smiling} on CelebA, gender as $\sens$), the frozen CLIP ViT-L/14 features with a trainable projection, and all hyperparameters including $\rhof$. We then resample the four $(\y,\sens)$ cells so that only $\Delta\pi$ changes, holding $\mathbb{P}(\y{=}1)$, $\mathbb{P}(\sens{=}1)$, and the sample size fixed. \Cref{tab:basegap} shows that \algname{} reduces $\Delta$DP by about half at every level, and by more at the smallest gap, while its AP cost grows with $\Delta\pi$.
\begin{table}[t]
\centering
\caption{Varying only the base-rate gap $\Delta\pi$ (frozen CLIP ViT-L/14, \texttt{Smiling}, 5 seeds).}
\label{tab:basegap}
\small
\setlength{\tabcolsep}{4pt}
\begin{tabular}{c cc cc c}
\toprule
\multirow{2}{*}{$\Delta\pi$} & \multicolumn{2}{c}{AP $\uparrow$} & \multicolumn{2}{c}{$\Delta$DP $\downarrow$} & \multirow{2}{*}{AP cost $\downarrow$} \\
\cmidrule(lr){2-3} \cmidrule(lr){4-5}
 & ERM & \algname{} & ERM & \algname{} & \\
\midrule
5  & 97.5$\pm$0.0 & 96.6$\pm$0.3 & 4.0$\pm$0.2  & 0.7$\pm$0.5  & 0.9 \\
15 & 97.5$\pm$0.0 & 96.4$\pm$0.2 & 14.4$\pm$0.5 & 6.6$\pm$0.7  & 1.1 \\
25 & 97.6$\pm$0.1 & 95.3$\pm$0.3 & 24.7$\pm$0.3 & 12.6$\pm$1.1 & 2.3 \\
35 & 97.8$\pm$0.1 & 93.9$\pm$0.3 & 35.4$\pm$0.3 & 17.4$\pm$1.6 & 3.9 \\
45 & 98.0$\pm$0.1 & 91.3$\pm$0.4 & 45.7$\pm$0.3 & 20.3$\pm$1.9 & 6.7 \\
\bottomrule
\end{tabular}
\end{table}

\subsection{Hyperparameter Sensitivity}
\label{app:hp_sensitivity}

\Cref{fig:ablations} ablates each of the four loss weights $\lamf, \lams, \rhof, \rhos$ on CelebA-A while holding the others at their default values. Fairness improvements are driven primarily by the independence weight $\rhof$. $\Delta$DP, ABCC, and MCDP all decrease sharply as $\rhof$ grows, at a gradual utility cost in AP. Varying $\lamf$ has little effect on the metrics. The Sensitive Space weights $\lams$ and $\rhos$ both shrink $\dfair$. A larger $\lams$ also lowers $\Delta$DP and ABCC, at a cost in AP, while a larger $\rhos$ allocates more dimensions to $\zsens$ with almost no change in AP, which tightens the bound in Theorem~\ref{thm:dp-hsic}. The outcome-conditioned metrics $\Delta$EOpp and $\Delta$EOdd fall as $\rhof$ grows to its default value and rise slightly beyond it, and they rise with $\lams$.\looseness=-1

\begin{figure}[h]
    \centering
    \includegraphics[width=\linewidth]{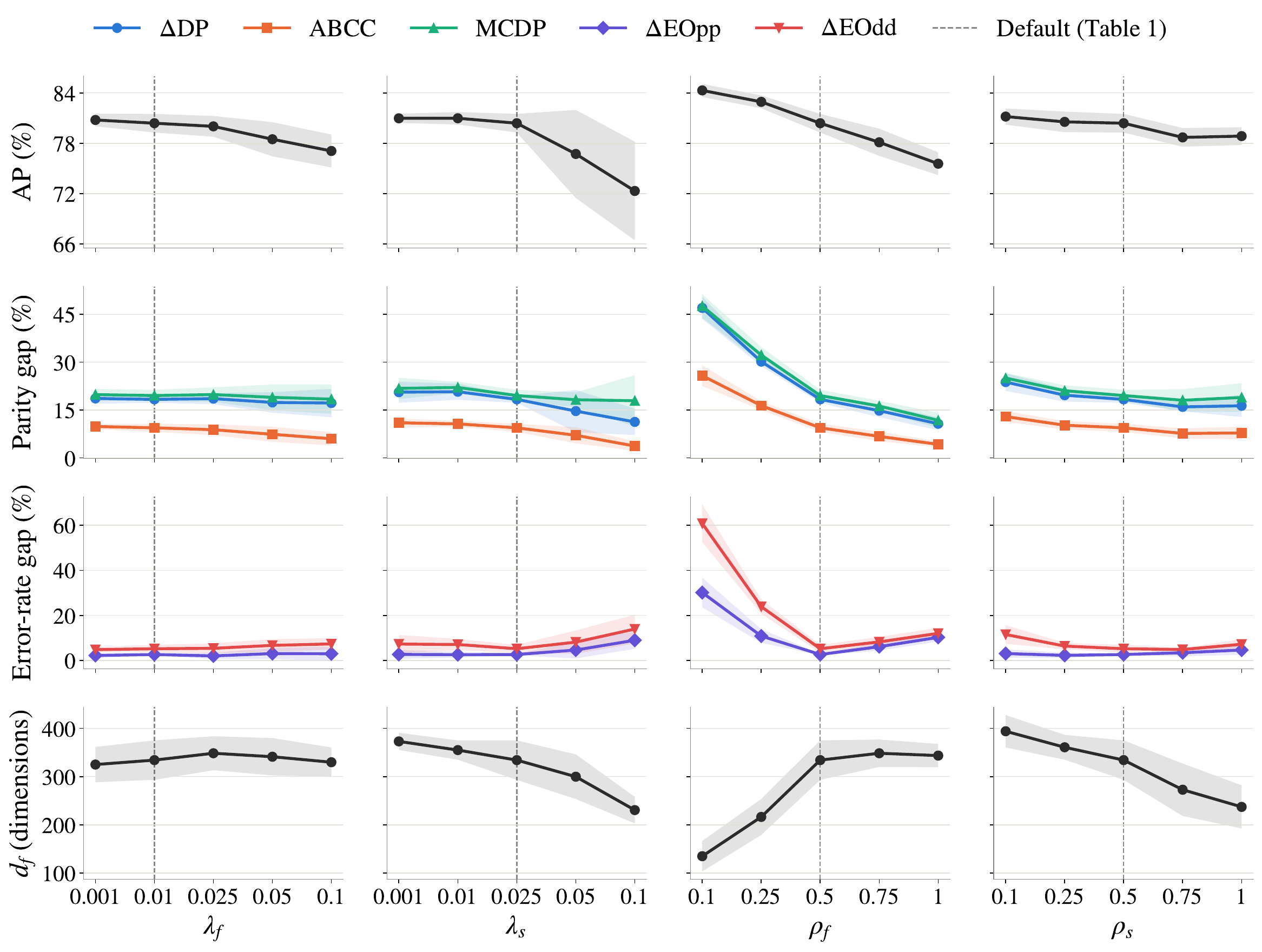}
    \caption{Hyperparameter ablation on CelebA-A (5 seeds). We vary $\lambda_f, \lambda_s, \rho_f, \rho_s$ (left to right) while fixing the other three at their defaults (dashed line). Rows: AP, parity gaps ($\Delta$DP, ABCC, MCDP), error-rate gaps ($\Delta$EOpp, $\Delta$EOdd), and the learned Fair Space dimensionality $\dfair$. Fairness improvements are driven primarily by $\rho_f$, and the method is robust to $\lambda_f$. Increasing $\lambda_s$ and $\rho_s$ shrinks $\dfair$.\looseness=-1}
    \label{fig:ablations}
\end{figure}

\subsection{Component Ablation Across Datasets}
\label{app:component_ablation_extra}

\Cref{tab:component_ablation_extra} extends \Cref{tab:component_ablation} to UTKFace-AR and reports utility and all five fairness metrics on both image datasets. The same progression observed on CelebA-A holds on UTKFace-AR. Introducing the mask without any Sensitive Space loss (Unregularized) already reduces $\Delta$DP from $26.6$ to $22.2$ on CelebA-A and from $17.8$ to $15.3$ on UTKFace-AR while shrinking $\dfair$, and regularizing the Sensitive Space recovers full \algname{}'s gains. Adding only $\lams$ or only $\rhos$ to Unregularized isolates each Sensitive Space term, and \algname{} without $\lamf$ removes the Fair Space compression. The two Sensitive Space terms contribute comparably, and neither alone reaches the $\Delta$DP of full \algname{}. Removing $\lamf$ matches full \algname{} on CelebA-A but costs $1.7$ F1 on UTKFace-AR. Removing $\rhof$ eliminates the fairness gains on both datasets, and on CelebA-A the Fair Space collapses to $25$ dimensions on average.

\begin{table}[h]
\centering
\caption{Component ablation on CelebA-A and UTKFace-AR (5 seeds). The utility metric is AP on CelebA-A and F1 on UTKFace-AR. The best value per metric within each dataset is in \textbf{bold}. $^{\dagger}$One of the five seeds collapses to a constant predictor.}
\label{tab:component_ablation_extra}
\small
\setlength{\tabcolsep}{4pt}
\resizebox{\linewidth}{!}{%
\begin{tabular}{l ccccccc}
\toprule
Variant & AP/F1$\uparrow$ & $\Delta$DP$\downarrow$ & ABCC$\downarrow$ & MCDP$\downarrow$ & $\Delta$EOpp$\downarrow$ & $\Delta$EOdd$\downarrow$ & $\dfair$ \\
\midrule
\multicolumn{8}{l}{\emph{CelebA-A}} \\
ERM                                 & \textbf{84.5$\pm$0.4} & 52.8$\pm$0.8 & 35.9$\pm$0.8 & 53.4$\pm$0.8 & 36.8$\pm$1.5 & 75.0$\pm$2.1 & 512 \\
No Sensitive Space                  & 80.7$\pm$0.7 & 26.6$\pm$1.4 & 16.1$\pm$1.0 & 26.8$\pm$1.0 & 5.5$\pm$0.9 & 18.4$\pm$2.8 & 512 \\
Unregularized ($\lams{=}\rhos{=}0$) & 80.1$\pm$0.8 & 22.2$\pm$2.4 & 12.4$\pm$1.7 & 22.8$\pm$1.6 & 1.9$\pm$0.9 & 9.1$\pm$2.4 & 391 \\
Unregularized $+\,\lams$            & 79.9$\pm$1.2 & 20.1$\pm$2.7 & 10.7$\pm$1.8 & 21.6$\pm$2.0 & 2.4$\pm$1.3 & 8.1$\pm$2.5 & 366 \\
Unregularized $+\,\rhos$            & 80.9$\pm$0.6 & 21.4$\pm$1.8 & 11.2$\pm$1.0 & 22.4$\pm$1.2 & \textbf{1.4$\pm$2.2} & 6.8$\pm$2.8 & 371 \\
\algname{} without $\lamf$          & 80.6$\pm$0.6 & 18.4$\pm$0.9 & 9.6$\pm$0.8 & \textbf{19.4$\pm$1.2} & 2.1$\pm$0.8 & \textbf{4.4$\pm$1.2} & 326 \\
\algname{} without $\rhof$\textsuperscript{$\dagger$} & 82.9$\pm$1.2 & 51.7$\pm$26.6 & 12.1$\pm$8.9 & 75.0$\pm$8.8 & 53.2$\pm$28.1 & 87.4$\pm$45.8 & 25 \\
\rowcolor{gray!15}\textbf{\algname{}} & 80.4$\pm$1.1 & \textbf{18.3$\pm$1.1} & \textbf{9.4$\pm$1.3} & 19.6$\pm$1.8 & 2.6$\pm$0.3 & 5.2$\pm$1.7 & 334 \\
\midrule
\multicolumn{8}{l}{\emph{UTKFace-AR}} \\
ERM                                 & \textbf{73.8$\pm$0.7} & 27.9$\pm$1.4 & 27.1$\pm$2.6 & 36.2$\pm$2.4 & 29.1$\pm$5.2 & 35.9$\pm$3.0 & 512 \\
No Sensitive Space                  & 71.8$\pm$1.6 & 17.8$\pm$2.0 & 15.7$\pm$1.9 & 23.3$\pm$1.9 & 18.5$\pm$4.5 & 24.3$\pm$5.8 & 512 \\
Unregularized ($\lams{=}\rhos{=}0$) & 71.1$\pm$1.1 & 15.3$\pm$1.1 & 10.9$\pm$0.9 & 19.9$\pm$1.3 & 19.3$\pm$6.1 & 23.9$\pm$6.0 & 389 \\
Unregularized $+\,\lams$            & 70.0$\pm$1.8 & 13.3$\pm$1.2 & \textbf{9.0$\pm$0.9} & \textbf{18.7$\pm$1.1} & 17.8$\pm$6.7 & 22.7$\pm$5.9 & 223 \\
Unregularized $+\,\rhos$            & 71.1$\pm$1.4 & 14.0$\pm$2.1 & 11.2$\pm$1.2 & 20.0$\pm$1.7 & 15.9$\pm$4.7 & \textbf{22.1$\pm$5.0} & 402 \\
\algname{} without $\lamf$          & 68.7$\pm$1.7 & 13.0$\pm$2.2 & 10.3$\pm$1.0 & 19.4$\pm$2.2 & 16.5$\pm$6.9 & 22.2$\pm$6.9 & 222 \\
\algname{} without $\rhof$          & \textbf{73.8$\pm$1.5} & 33.4$\pm$4.1 & 19.5$\pm$1.1 & 39.4$\pm$2.7 & 37.8$\pm$4.9 & 49.5$\pm$5.1 & 270 \\
\rowcolor{gray!15}\textbf{\algname{}} & 70.4$\pm$1.5 & \textbf{12.6$\pm$1.1} & 9.1$\pm$1.0 & \textbf{18.7$\pm$1.2} & \textbf{14.9$\pm$4.6} & \textbf{22.1$\pm$4.9} & 233 \\
\bottomrule
\end{tabular}%
}
\end{table}

\subsection{HSIC Kernel and Bandwidth Sensitivity}
\label{app:kernel_sensitivity}

The HSIC regularizer $\lossreg$ can be instantiated with any characteristic kernel. We test whether \algname{} is empirically stable across kernel choices. All experiments are on CelebA-A with otherwise identical hyperparameters, averaged over 5 seeds.

\paragraph{Kernel choice.} \Cref{tab:kernel_family} compares three characteristic kernel families (Gaussian RBF, Mat\'ern $3/2$, Laplacian). Gaussian and Mat\'ern yield nearly identical results, suggesting stability across smooth characteristic kernels. The Laplacian kernel gives weaker fairness at the same regularization strength, but still reduces $\Delta$DP by $33\%$ relative to ERM.

\paragraph{Bandwidth sensitivity.} For the Gaussian RBF kernel, we sweep the bandwidth $\sigma \in \{2, 3, 4, 5, 10\}$ on CelebA-A (\Cref{tab:kernel_family}). Every bandwidth reduces $\Delta$DP by $53\%$ to $65\%$ relative to ERM. We fix $\sigma = 5$ for image and $\sigma = 2$ for tabular experiments (\Cref{tab:fairsplit-params}).

\begin{table}[t]
\centering
\caption{HSIC kernel family and Gaussian RBF bandwidth on CelebA-A (5 seeds). Only the kernel or $\sigma$ changes, and ERM is reported as a reference without fairness regularization.}
\label{tab:kernel_family}
\small
\begin{tabular}{lccc}
\toprule
 & AP$\uparrow$ & $\Delta$DP$\downarrow$ & ABCC$\downarrow$ \\
\midrule
ERM & 84.5$\pm$0.4 & 52.8$\pm$0.8 & 35.9$\pm$0.8 \\
\midrule
\multicolumn{4}{l}{\emph{Kernel family ($\sigma = 5$)}} \\
Gaussian RBF   & 80.4$\pm$1.1 & 18.3$\pm$1.1 & \phantom{0}9.4$\pm$1.3 \\
Mat\'ern $3/2$ & 80.5$\pm$1.5 & 18.6$\pm$2.9 & 10.0$\pm$2.1 \\
Laplacian      & 78.1$\pm$2.5 & 35.4$\pm$6.6 & 16.4$\pm$4.2 \\
\midrule
\multicolumn{4}{l}{\emph{Gaussian RBF bandwidth}} \\
$\sigma = 2$  & 80.8$\pm$1.5 & 24.7$\pm$2.4 & 13.5$\pm$1.5 \\
$\sigma = 3$  & 80.7$\pm$0.9 & 19.9$\pm$1.9 & 10.6$\pm$1.6 \\
$\sigma = 4$  & 81.0$\pm$1.0 & 19.9$\pm$1.0 & 10.4$\pm$0.6 \\
$\sigma = 5$  & 80.4$\pm$1.1 & 18.3$\pm$1.1 & \phantom{0}9.4$\pm$1.3 \\
$\sigma = 10$ & 79.3$\pm$1.4 & 20.8$\pm$1.6 & \phantom{0}9.8$\pm$1.5 \\
\bottomrule
\end{tabular}
\end{table}

\subsection{Computational Cost}
\label{app:compute_cost}

We measure wall-clock training time per run on a single NVIDIA A100, averaged over 5 seeds, for both a tabular (Adult) and an image (CelebA-A) dataset. \Cref{tab:compute_cost} reports absolute seconds and the relative cost over the ERM baseline. \algname{} adds $17$ to $22\%$ wall-clock overhead over ERM, which lies within the range of the other in-processing fairness methods we evaluate ($0.93$ to $1.76\times$). On Adult, the overhead comes primarily from the HSIC computation in Equation~\ref{eq:hsic_loss}. On CelebA-A, it comes mostly from initializing the mask, which requires one pass over the training set, while the per-step HSIC and mask updates (\Cref{prop:mask_solution}) add only about $3\%$.

\begin{table}[t]
\centering
\caption{Wall-clock training time per run (in seconds), averaged over 5 seeds on a single NVIDIA A100. ``vs.\ ERM'' is the multiplicative cost relative to the ERM baseline.}
\label{tab:compute_cost}
\small
\begin{tabular}{lcccc}
\toprule
Method & Adult (s) & vs.\ ERM & CelebA-A (s) & vs.\ ERM \\
\midrule
ERM                     & 11.4 & 1.00$\times$ & 27.7 & 1.00$\times$ \\
\midrule
AdvDebias               & 10.6 & 0.93$\times$ & 28.6 & 1.03$\times$ \\
FairMixup               & 12.4 & 1.09$\times$ & 35.8 & 1.29$\times$ \\
DRAlign                 & 13.8 & 1.21$\times$ & 48.6 & 1.76$\times$ \\
HSIC$(Y,S)$             & 11.9 & 1.05$\times$ & 28.4 & 1.03$\times$ \\
DiffMCDP                & 16.3 & 1.43$\times$ & 28.8 & 1.04$\times$ \\
PRemover                & 11.3 & 0.99$\times$ & 29.8 & 1.08$\times$ \\
MMD-B-Fair              & 16.6 & 1.46$\times$ & 35.5 & 1.28$\times$ \\
\midrule
\rowcolor{gray!15}\textbf{\algname{}} & 13.9 & 1.22$\times$ & 32.3 & 1.17$\times$ \\
\bottomrule
\end{tabular}
\end{table}

\subsection{Robustness to Noisy and Missing Sensitive Attributes}
\label{app:noisy_missing}
\algname{} requires sensitive attributes during training, but in practice, these labels can be partially incorrect or unavailable. We assess robustness on CelebA-A under two corruption regimes, without changing any hyperparameters. \emph{Noisy~$p$:} each sensitive label is independently flipped with probability $p$. \emph{Missing~$p$:} a fraction $p$ of training samples have their sensitive label dropped entirely. For these samples, the HSIC and structural-compression terms involving $\sens$ are masked out, while the cross-entropy task loss is still applied. As a control, we also train with randomly shuffled sensitive labels. Results are averaged over 5 seeds.
As shown in \Cref{tab:noisy_missing}, \algname{} tolerates missing sensitive labels well: with $10\%$ and $30\%$ of the labels removed, $\Delta$DP is $21.8$ and $28.0$ against $52.8$ for ERM, and the gain shrinks gradually as more labels are removed. It is more sensitive to incorrect labels: $10\%$ label noise keeps most of the gain ($\Delta$DP $28.3$), while $30\%$ removes it. With shuffled labels, \algname{} matches ERM, so its gains come from the sensitive labels. Extending \algname{} to the semi-supervised setting is left for future work.
\begin{table}[t]
\centering
\caption{Robustness on CelebA-A under noisy and missing sensitive attributes (5 seeds). ``Noisy $p$'' flips each sensitive label with probability $p$. ``Missing $p$'' drops the sensitive label for a fraction $p$ of samples (these samples contribute only to the cross-entropy loss). ``Shuffled'' randomly permutes all sensitive labels. ERM and clean \algname{} are reported as references.}
\label{tab:noisy_missing}
\small
\begin{tabular}{lccc}
\toprule
Condition & AP$\uparrow$ & $\Delta$DP$\downarrow$ & ABCC$\downarrow$ \\
\midrule
ERM                 & 84.5$\pm$0.4 & 52.8$\pm$0.8 & 35.9$\pm$0.8 \\
\algname{} (clean)  & 80.4$\pm$1.1 & 18.3$\pm$1.1 & \phantom{0}9.4$\pm$1.3 \\
\midrule
Noisy 10\%          & 82.8$\pm$0.7 & 28.3$\pm$2.6 & 16.2$\pm$1.5 \\
Noisy 30\%          & 85.5$\pm$0.2 & 53.1$\pm$0.9 & 31.5$\pm$0.9 \\
Noisy 50\%          & 85.6$\pm$0.3 & 54.0$\pm$0.4 & 32.2$\pm$0.6 \\
\midrule
Missing 10\%        & 81.0$\pm$0.9 & 21.8$\pm$1.9 & 11.6$\pm$1.2 \\
Missing 30\%        & 81.4$\pm$1.2 & 28.0$\pm$1.3 & 15.4$\pm$1.0 \\
Missing 50\%        & 83.3$\pm$0.7 & 36.7$\pm$1.7 & 22.5$\pm$0.9 \\
Missing 70\%        & 84.4$\pm$0.4 & 46.5$\pm$1.1 & 29.4$\pm$0.4 \\
\midrule
Shuffled            & 85.4$\pm$0.2 & 53.4$\pm$0.9 & 31.8$\pm$0.6 \\
\bottomrule
\end{tabular}
\end{table}

\clearpage

\input{checklist.tex}

\end{document}

%% file: abbreviations.tex
\newcommand{\vct}[1]{\mathbf{#1}}
\newcommand{\mat}[1]{\mathbf{#1}}
\newcommand{\set}[1]{\mathcal{#1}}

\newcommand*{\algname}{\textsc{FairSplit}}

\newcommand{\paramEnc}{\phi}
\newcommand{\paramClf}{\psi}
\newcommand{\paramAll}{\theta}

\newcommand{\tagSens}{s}
\newcommand{\tagFair}{f}

\newcommand{\symbX}{x}
\newcommand{\symbY}{y}
\newcommand{\symbZ}{z}
\newcommand{\symbS}{s}

\newcommand{\xspace}{\set{X}}
\newcommand{\yspace}{\set{Y}}
\newcommand{\zspace}{\set{Z}}
\newcommand{\sspace}{\set{S}}

\newcommand{\vx}{\vct{\symbX}}
\newcommand{\vy}{\vct{\symbY}}
\newcommand{\vz}{\vct{\symbZ}}
\newcommand{\vs}{\vct{\symbS}}
\newcommand{\vb}{\vct{b}}

\newcommand{\y}{\symbY}
\newcommand{\sens}{\symbS}
\newcommand{\nngroups}{G}

\newcommand{\encparams}{\paramEnc}
\newcommand{\clfparams}{\paramClf}
\newcommand{\allparams}{\paramAll}

\newcommand{\encoder}{f_{\paramEnc}}
\newcommand{\classifier}{g_{\paramClf}}
\newcommand{\model}{h_{\paramAll}}

\newcommand{\zfair}{\vz_{\tagFair}}
\newcommand{\zsens}{\vz_{\tagSens}}

\newcommand{\Zfair}{\mat{Z}_{\tagFair}}
\newcommand{\Zsens}{\mat{Z}_{\tagSens}}

\newcommand{\dfair}{d_{\tagFair}}

\newcommand{\cclass}{\boldsymbol{\mu}}
\newcommand{\cgroup}{\boldsymbol{\nu}}

\newcommand{\meanemb}{\boldsymbol{\xi}}

\newcommand{\loss}{\mathcal{L}}
\newcommand{\losstask}{\loss_{\text{task}}}
\newcommand{\losscomp}{\loss_{\text{comp}}}
\newcommand{\lossreg}{\loss_{\text{reg}}}
\newcommand{\losstotal}{\loss_{\text{total}}}

\newcommand{\lamf}{\lambda_{\tagFair}}
\newcommand{\lams}{\lambda_{\tagSens}}
\newcommand{\rhof}{\rho_{\tagFair}}
\newcommand{\rhos}{\rho_{\tagSens}}

\newcommand{\vartagClass}{y}
\newcommand{\vartagGroup}{s}

\DeclareMathOperator{\HSIC}{HSIC}
\newcommand{\ddp}{\Delta_{\text{DP}}}

\newcommand{\real}{\mathbb{R}}
\newcommand{\expect}{\mathbb{E}}
\newcommand{\norm}[1]{\left\|#1\right\|}

%% file: checklist.tex
\section*{NeurIPS Paper Checklist}

\begin{enumerate}

\item {\bf Claims}
    \item[] Question: Do the main claims made in the abstract and introduction accurately reflect the paper's contributions and scope?
    \item[] Answer: \answerYes{} 
    \item[] Justification: Our theoretical claim is supported with a proof: we bound a continuous relaxation of the demographic parity gap by the HSIC between the Fair Space and the sensitive attribute and the Fair Space dimensionality, both of which \algname{} controls during training. Empirically, we validate the method on standard fair classification benchmarks, including multiclass and non-binary sensitive-attributes.


\item {\bf Limitations}
    \item[] Question: Does the paper discuss the limitations of the work performed by the authors?
    \item[] Answer: \answerYes{} 
    \item[] Justification: We provide a dedicated \emph{Limitations and Discussion} subsection (\Cref{sec:limitations}) covering the linear-head assumption of our bound, the bound's $1/p_{\min}$ dependence on minimum group prevalence, the requirement for sensitive attributes during training, and the residual sensitive information retained in the Sensitive Space. The first three are paired with targeted experiments and pointers to the relevant appendix. The last is an inherent design property and is addressed by a deployment recommendation.

\item {\bf Theory assumptions and proofs}
    \item[] Question: For each theoretical result, does the paper provide the full set of assumptions and a complete (and correct) proof?
    \item[] Answer: \answerYes{} 
    \item[] Justification: The paper has two main theoretical results: the closed-form mask (Proposition~\ref{prop:mask_solution}) and the demographic parity bound (Theorem~\ref{thm:dp-hsic}). Both state their assumptions explicitly in the theorem statements and have complete proofs in Appendix~\ref{app:proof_mask} and Appendix~\ref{app:proof_main_theorem}.

    \item {\bf Experimental result reproducibility}
    \item[] Question: Does the paper fully disclose all the information needed to reproduce the main experimental results of the paper to the extent that it affects the main claims and/or conclusions of the paper (regardless of whether the code and data are provided or not)?
    \item[] Answer: \answerYes{}
    \item[] Justification: We describe the experimental setup in \Cref{sec:experiments}. \Cref{app:reproducibility} covers the datasets, metrics, implementation, software and hardware setup, and the used hyperparameters for every method. \Cref{app:datasets} also provides per-dataset licenses and download sources. Our code, including all training configurations, is publicly available at \codelink{}.

\item {\bf Open access to data and code}
    \item[] Question: Does the paper provide open access to the data and code, with sufficient instructions to faithfully reproduce the main experimental results, as described in supplemental material?
    \item[] Answer: \answerYes{}
    \item[] Justification: Our code is publicly available at \codelink{}, with instructions to reproduce the main experiments for \algname{} and all baselines. \Cref{app:datasets} provides per-dataset licenses and download sources, and \Cref{app:reproducibility} documents the environment, the used hyperparameters, and the experimental protocol.

\item {\bf Experimental setting/details}
    \item[] Question: Does the paper specify all the training and test details (e.g., data splits, hyperparameters, how they were chosen, type of optimizer) necessary to understand the results?
    \item[] Answer: \answerYes{} 
    \item[] Justification: \Cref{sec:experiments} specifies datasets, baselines, evaluation metrics, the five-seed protocol, and the model-selection rule. \Cref{app:reproducibility} gives the optimizer, learning-rate schedule, encoder architectures, batch sizes, training length, and the per-method hyperparameters.

\item {\bf Experiment statistical significance}
    \item[] Question: Does the paper report error bars suitably and correctly defined or other appropriate information about the statistical significance of the experiments?
    \item[] Answer: \answerYes{} 
    \item[] Justification: All main tables report mean $\pm$ one standard deviation across five seeds (\Cref{sec:experiments}), with seed variation covering the train/val/test split, model initialization, and batch order. We do not perform dedicated significance tests.

\item {\bf Experiments compute resources}
    \item[] Question: For each experiment, does the paper provide sufficient information on the computer resources (type of compute workers, memory, time of execution) needed to reproduce the experiments?
    \item[] Answer: \answerYes{}
    \item[] Justification: \Cref{app:hardware} reports the hardware. \Cref{app:compute_cost} provides per-run wall-clock training times for \algname{} and all baselines on Adult and CelebA-A, averaged over 5 seeds.
    
\item {\bf Code of ethics}
    \item[] Question: Does the research conducted in the paper conform, in every respect, with the NeurIPS Code of Ethics \url{https://neurips.cc/public/EthicsGuidelines}?
    \item[] Answer: \answerYes{}
    \item[] Justification: We read and followed the NeurIPS Code of Ethics. The work targets fairness directly (mitigating disparate treatment based on sensitive attributes), uses only established public benchmarks under their stated terms (\Cref{app:datasets}) and involves no human-subjects experimentation.

\item {\bf Broader impacts}
    \item[] Question: Does the paper discuss both potential positive societal impacts and negative societal impacts of the work performed?
    \item[] Answer: \answerYes{}
    \item[] Justification: The work targets a positive societal application (mitigating disparate treatment based on sensitive attributes), discussed in \Cref{sec:intro} and \Cref{sec:related_work}. Specific risks and limitations are addressed in \Cref{sec:limitations}: the Sensitive Space retains protected-attribute information so model parameters must be secured, and the method requires sensitive labels at training time, with partial robustness to noisy or missing labels shown in \Cref{app:noisy_missing}.
    
\item {\bf Safeguards}
    \item[] Question: Does the paper describe safeguards that have been put in place for responsible release of data or models that have a high risk for misuse (e.g., pre-trained language models, image generators, or scraped datasets)?
    \item[] Answer: \answerNA{} 
    \item[] Justification: We do not release pre-trained models or data.

\item {\bf Licenses for existing assets}
    \item[] Question: Are the creators or original owners of assets (e.g., code, data, models), used in the paper, properly credited and are the license and terms of use explicitly mentioned and properly respected?
    \item[] Answer: \answerYes{} 
    \item[] Justification: All datasets are cited and their licenses, source URLs, and terms of use are documented in \Cref{app:datasets}. Code we build upon (FFB~\citep{ffb}, DiffMCDP~\citep{DiffMCDP_MCDP_ICML}, H-SPLID~\citep{H-SPLID_NeurIPS}, MMD-B-Fair~\citep{mmd-b-fair}) is cited with repository links in \Cref{app:reproducibility}, and we use these assets in accordance with their stated terms.

\item {\bf New assets}
    \item[] Question: Are new assets introduced in the paper well documented and is the documentation provided alongside the assets?
    \item[] Answer: \answerYes{}
    \item[] Justification: We release our code at \codelink{} with a README documenting setup, dependencies, and commands to reproduce the main experiments.

\item {\bf Crowdsourcing and research with human subjects}
    \item[] Question: For crowdsourcing experiments and research with human subjects, does the paper include the full text of instructions given to participants and screenshots, if applicable, as well as details about compensation (if any)? 
    \item[] Answer: \answerNA{} 
    \item[] Justification: No research with human subjects or crowdsourcing.

\item {\bf Institutional review board (IRB) approvals or equivalent for research with human subjects}
    \item[] Question: Does the paper describe potential risks incurred by study participants, whether such risks were disclosed to the subjects, and whether Institutional Review Board (IRB) approvals (or an equivalent approval/review based on the requirements of your country or institution) were obtained?
    \item[] Answer: \answerNA{}
    \item[] Justification: No research with human subjects or crowdsourcing.

\item {\bf Declaration of LLM usage}
    \item[] Question: Does the paper describe the usage of LLMs if it is an important, original, or non-standard component of the core methods in this research? Note that if the LLM is used only for writing, editing, or formatting purposes and does \emph{not} impact the core methodology, scientific rigor, or originality of the research, declaration is not required.
    \item[] Answer: \answerNA{} 
    \item[] Justification: LLMs were only used for writing, editing, formatting, and coding assistance.

\end{enumerate}